\documentclass{article}

\usepackage[preprint]{neurips_2023}

\usepackage{graphicx}
\usepackage{amsmath}
\usepackage{amsthm}
\setcitestyle{numbers,square,comma}
\usepackage{fancyhdr}
\usepackage[colorlinks,
            linkcolor=red,       
            anchorcolor=blue,  
            citecolor=blue, 
]{hyperref}
\usepackage{algorithm}
\usepackage{algorithmic}
\usepackage{tabu}
\usepackage{wrapfig,lipsum,booktabs}
\usepackage{amssymb}

\usepackage[utf8]{inputenc} 
\usepackage[T1]{fontenc}    
\usepackage{hyperref}       
\usepackage{url}            
\usepackage{booktabs}       
\usepackage{amsfonts}       
\usepackage{nicefrac}       
\usepackage{microtype}      
\usepackage{xcolor}         

\usepackage{dsfont}

\usepackage{subcaption}

\usepackage{setspace}

\newtheorem{theorem}{Theorem}[section]

\newtheorem{lemma}{Lemma}[section]

\def\@fnsymbol#1{\ensuremath{\ifcase#1\or *\or \dagger\or \ddagger\or
   \mathsection\or \mathparagraph\or \|\or **\or \dagger\dagger
   \or \ddagger\ddagger \else\@ctrerr\fi}}

\title{On the Value of Myopic Behavior in Policy Reuse}

\author{%
  Kang Xu$^{\text{1}~\text{2}}$\thanks{This work was done during Kang Xu's internship at Shanghai Artificial Intelligence Laboratory.}\qquad
  Chenjia Bai$^\text{2}$
  \thanks{Correspondence to Chenjia Bai <baichenjia@pjlab.org.cn>, Wei Li <fd$\_$liwei@fudan.edu.cn>.}\qquad
  Shuang Qiu$^\text{3}$\qquad
  Haoran He$^\text{4}$\qquad
  Bin Zhao$^{\text{2}~\text{5}}$\qquad\\
  \AND
  \textbf{Zhen Wang$^\text{5}$}\qquad
  \textbf{Wei Li$^{\text{1}\dagger}$}\qquad
  \textbf{Xuelong Li$^{\text{2}~\text{5}}$}\qquad
  \AND
  \textnormal{$^\text{1}~$Fudan University}\quad
  \textnormal{$^\text{2}~$Shanghai Artificial Intelligence Laboratory}\quad
  \textnormal{$^\text{3}~$HKUST}\\
  \textnormal{$^\text{4}~$Shanghai Jiao Tong University}\quad
  \textnormal{$^\text{5}~$Northwestern Polytechnical University}
}

\begin{document}

\maketitle

\begin{abstract}
    Leveraging learned strategies in unfamiliar scenarios is fundamental to human intelligence. In reinforcement learning, rationally reusing the policies acquired from other tasks or human experts is critical for tackling problems that are difficult to learn from scratch. In this work, we present a framework called Selective Myopic bEhavior Control~(SMEC), which results from the insight that the short-term behaviors of prior policies are sharable across tasks. By evaluating the behaviors of prior policies via a hybrid value function architecture, SMEC adaptively aggregates the sharable short-term behaviors of prior policies and the long-term behaviors of the task policy, leading to coordinated decisions. Empirical results on a collection of manipulation and locomotion tasks demonstrate that SMEC outperforms existing methods, and validate the ability of SMEC to leverage related prior policies.
\end{abstract}

\section{Introduction}
\label{sec:intro}
Reinforcement learning has demonstrated a wide range of successes~\cite{akkaya2019solving,lee2020learning,berner2019dota} by learning from scratch. While effective, a major challenge is the need for agents to acquire extensive experience, which can be costly and time-consuming, especially in real-world scenarios. Contrary to learning without prior knowledge, human intelligence can quickly identify the relationships between the current task and previous experience, thereby facilitating the completion of novel tasks by deploying learned strategies. Building upon this observation, we focus on policy reuse with a collection of prior policies for efficient learning in downstream tasks~\cite{fernandez2006probabilistic,rosman2016bayesian}.

Intuitively, more prior policies could lead to more efficient learning by utilizing the abundant knowledge of the prior policies. However, reusing the policies without knowing their properties can be non-trivial, since some policies can provide irrelevant or even harmful behaviors concerning the current task. Thus, how to \textit{efficiently identify the reusable policies} and \textit{rationally exploit the policies} are the essential problems of policy reuse. Previous policy reuse algorithms broadly fall into three categories: advantage-based, aggregation-based, and behavior-based methods. Advantage-based algorithms~\cite{cheng2020policy,zhangcup} exploit the one-step advantage induced by the advised actions from prior policies for policy regularization. Aggregation-based algorithms~\cite{laroche2017multi,gimelfarb2021contextual,barekatain2021multipolar} compose actions from all prior policies via learning the mixture functions. Unlike the first two categories, behavior-based algorithms~\cite{li2019context,vezzani2022skills,yang2020efficient,kurenkov2020ac} utilize the temporally-extended behaviors of prior policies to guide online interactions, which makes them particularly appealing, as the agent can deploy the advantageous actions from prior policies. In addition, the independent task policy that does not build on prior policies is computationally practical. However, existing behavior-based methods assume access to related prior policies concerning the current task, limiting the generality of these methods.

\begin{figure}[t]
    \centering
    \includegraphics[width=0.95\textwidth]{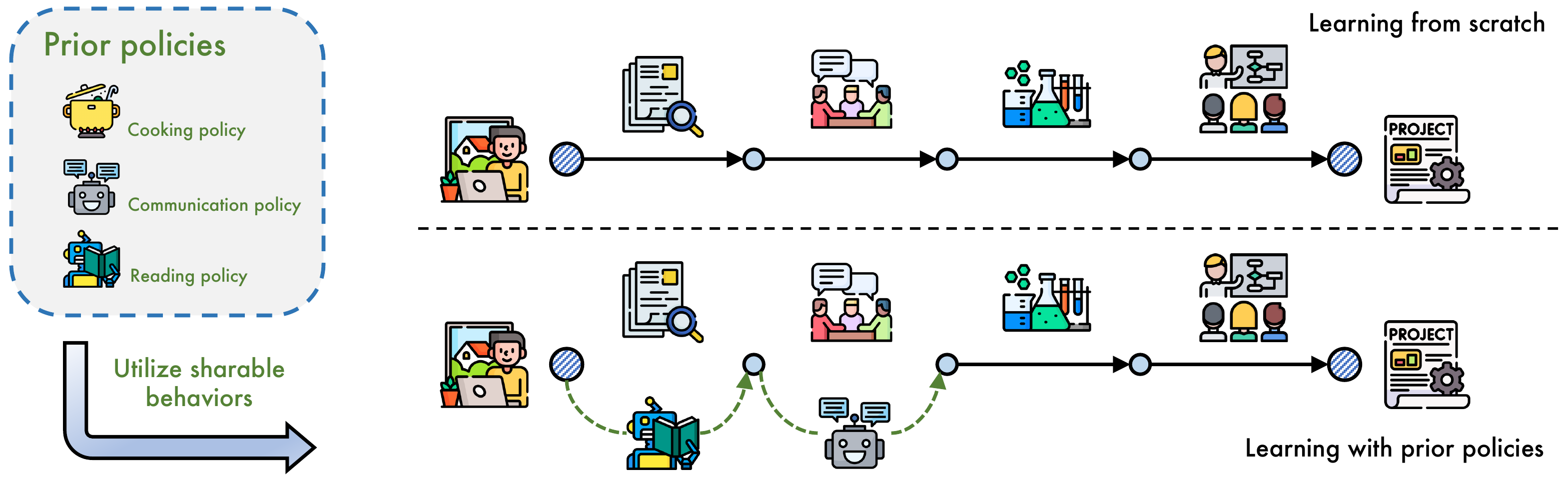}
    \caption{
    Contrary to accomplishing tasks~(\textit{i.e.}, a project) by learning from scratch, human typically decompose the task into multiple stages and deploy learned policies to solve the corresponding sub-tasks. Inspired by this, we propose that the sharable short-term behaviors of prior policies concerning the current task can be identified and exploited for efficient learning.
    }
    \vspace{-1.em}
    \label{fig:behavior_sharing_illustration}
\end{figure}

Imagine we are solving a project which can be decomposed into multiple stages~(\textit{i.e.},~survey, discussion, experiments, and presentation) as shown in Figure~\ref{fig:behavior_sharing_illustration}, the policies learned in previous tasks such as reading and communication will be quickly identified and exploited for the corresponding short-term sub-tasks, even if the prior policies are not directly relevant to the current goal. 
Motivated by this, our key insight is that \textit{the short-term behaviors of prior policies are sharable across tasks}. To operationalize the idea, we propose \textbf{Se}lective \textbf{M}yopic b\textbf{E}havior \textbf{C}ontrol~(\textbf{SMEC}), which adaptively exploits these short-term behaviors to facilitate learning. Specifically, SMEC switches between policies for short-term interactions and adaptively utilizes the beneficial behaviors of prior policies by evaluating their short-term performance. To select the most effective policy for the subsequent interactions, SMEC compares the value estimations of \textit{long-term} task policy behaviors and \textit{short-term} prior policy behaviors at each switch point. We propose a hybrid value function architecture that evaluates behaviors across all policies, enabling scalability to a large number of prior policies. By identifying and utilizing the beneficial short-term behaviors, SMEC guides the online interactions at the early training stage, accelerating the learning process. As the training proceeds, prior policies are automatically weaned off due to the improved values of the task policy behaviors, which circumvents the sub-optimal behaviors of prior policies to hinder training.

We summarize the main contributions of our approach as follows: (1)~We highlight that the short-term behaviors are sharable across tasks, which can be leveraged in policy reuse. (2)~We propose a simple yet efficient algorithm termed Selective Myopic bEhavior Control~(SMEC) that performs behavior planning via evaluating the short-term behaviors of prior policies~(Section~\ref{sec:method}). (3)~We verify the effectiveness of the proposed approach by conducting extensive experiments on various tasks, with comparison to several baseline methods~(Section~\ref{sec:exp}).

\section{Preliminaries and Formulation}
\label{sec:preliminary}

\textbf{Markov decision processes~(MDPs).}~
A Markov decision process~(MDP)~\cite{howard1960dynamic} is specified by a state space $\mathcal{S}$, action space $\mathcal{A}$, transition probabilities $\mathcal{P}: \mathcal{S}\times\mathcal{A}\rightarrow \Delta(\mathcal{S})$, reward function $r\in [0, R_{\rm max}]: \mathcal{S}\times\mathcal{A}\rightarrow \mathbb{R}$, initial state distribution $d_0: \mathcal{S}\rightarrow \mathbb{R}$, and a discount factor $\gamma\in [0,1)$. In this paper, we focus on the infinite-horizon case,  where an agent interacts with the environment using a policy $\pi \in \Pi: \mathcal{S}\rightarrow \Delta(\mathcal{A})$ to generate the trajectories $\tau:= (s_0, a_0,s_1,a_1,\dots)$ with distribution $\mathrm{P}^\pi(\tau)=d_0(s_0)\prod_{t=0}^{\infty}\pi(a_t|s_t)\mathcal{P}(s_{t+1}|s_t,a_t)$. The performance of the policy is quantified as the expectation of the discounted return $J_\gamma(\pi):= \mathbb{E}_{\mathrm{P}^\pi(\tau)}\left[\sum_{t=0}^\infty \gamma^t r(s_t,a_t)\right]$. The goal is to find the optimal policy $\pi^* := \arg\max_{\pi\in\Pi}~\mathbb{E}_{\mathrm{P}^\pi(\tau)}\left[\sum_{t=0}^\infty \gamma^t r(s_t,a_t)\right]$. 

\textbf{Visitation distributions and value functions.}~
We denote $\mathbb{P}^\pi_t: \mathcal{S}\rightarrow [0, 1]$ as the state distribution at time $t$ induced by the policy $\pi$ starting from the initial state distribution and define the discounted state visitation distribution $d^\pi(s):= (1-\gamma)\mathbb{E}\left[\sum_{t=0}^\infty \gamma^t \mathbb{P}^\pi_t(s) | s_0 \sim d_0(\cdot)\right]$. Similarly, we define the discounted state-action visitation distribution as $\rho^\pi(s, a):=(1-\gamma)\mathbb{E}\left[\sum_{t=0}^\infty \gamma^t \mathbb{P}^\pi_t(s)\pi(a|s) | s_0 \sim d_0(\cdot)\right]$. The value function $V_\gamma^\pi(s):=\mathbb{E}\left[\sum_{t=0}^{\infty}\gamma^t r(s_t,a_t)|s_0 = s \right]$ and the state action value function $Q^\pi_\gamma(s,a):=\mathbb{E}\left[\sum_{t=0}^{\infty}\gamma^t r(s_t,a_t)|s_0 = s, a_0 = a\right]$ quantify the expected return induced by the policy $\pi$ starting from certain state or state-action pair, respectively. Due to the large state and action space in modern problems, existing works typically introduce the function approximators such as neural networks to estimate the value functions~\cite{van2016deep,van2019use,schulman2017proximal,haarnoja2018soft}, \textit{e.g.}, $Q_\theta$ with parameters $\theta\in \mathbb{R}^d$.

\textbf{Problem formulation.}~
Throughout this work, we aim to optimize a task-specific policy $\pi$ concerning the current task $\mathcal{M}$, with the assistance of the prior policies $\{\mu_i: \mathcal{S}\rightarrow\mathcal{A}\}_{i=1}^K$, hoping to achieve sample-efficient learning. Concretely, we assume the prior policies are learned from the shifted MDP with different reward functions $\{r_i\}_{i=1}^K$ or transition probabilities $\{\mathcal{P}_i\}_{i=1}^K$, and the performance of the prior policies in the current task is unknown. 
Unlike previous works that only consider different reward functions (\textit{e.g.}, Meta RL~\cite{rakelly2019efficient}), we extend our analysis to include different dynamics, ensuring the generality of our setting and investigating the universality of our approach. 

\section{Method}
\label{sec:method}

To rationally exploit the prior policies, it is crucial to evaluate the behaviors of prior policies in the current task and utilize their beneficial behaviors. To accomplish these goals, we first propose to evaluate the prior policies concerning their short-term behaviors~(Section~\ref{sec:method:evaluation}). Then we introduce value-guided behavior planning based on the evaluation of prior policies~(Section~\ref{sec:method:exploit}). Furthermore, we introduce theoretical analysis for the induced behavior policy~(Section~\ref{sec:method:theory}). For the pseudocode of our method, please refer to Algorithm~\ref{alg:pseudocode} in Appendix~\ref{app:pseudocode}.

\subsection{Evaluate Prior Policies Myopically}
\label{sec:method:evaluation}
Starting from the insight that the beneficial short-term behaviors of the prior policies can assist in accomplishing the current task, we need to estimate the returns of all prior policy behaviors within a short horizon during the online interactions. For this goal, the common option is to perform planning with a world model that can be learned with the collected transitions~\cite{chua2018deep,sekar2020planning,hafner2019learning}. However, the number of prior policies scales the computation cost induced by planning with all policies. In addition, the distribution shift between the training data for the world model and the actions induced by the prior policies can incur spurious evaluation. 

From another perspective, the value function learned with the collected data provides the expected return estimation of the behavior policy within a specific horizon~\cite{sutton2018reinforcement}. Reducing the discount factor used for the value estimation will shorten the horizons that the value function considers. Consider that we limit the length of the short-term behaviors to $h$ and perform the evaluation that only concerns the behaviors within $h$ future steps, we propose to use a truncated behavior discount factor $\bar{\gamma}$ that satisfies $\bar{\gamma}^h \approx 0$ to achieve the truncated behavior evaluation. In practice, by using a constant $\epsilon \gtrapprox 0$, we define the discount factor that truncates the values after $h$ steps as $\bar{\gamma}:=\epsilon^{\frac{1}{h}}$. 

Given a collection of prior policies $\{\mu_i\}_{i=1}^K$, we aim to evaluate the short-term behaviors of each prior policy with respect to the current task. Let $Q^{\mu_i}_{\bar{\gamma}}(s,a): \mathcal{S}\times\mathcal{A}\rightarrow \mathbb{R},~i\in[1,K]$ denotes the short-term value function concerning the prior policy $\mu_i$. During training, we perform temporal-difference learning~\cite{munos2008finite} that is widely used in modern off-policy algorithms~\cite{fujimoto2018addressing,haarnoja2018soft}.
However, the computation cost required for the value estimation depends on the number of prior policies. If the number of prior policies is large, training an individual value function for each prior policy will be computationally expensive. To address the problem, we propose an aggregated value function architecture that simultaneously performs value estimation over all policies, as shown in Figure~\ref{fig:method_semantic}~(Left). Let $Q_{\theta}:\mathcal{S}\times\mathcal{A}\rightarrow \mathbb{R}^{1+K}$ denotes the aggregated value function and $\{Q_{\theta}^{\pi,\gamma}(s, a)\}\cup\{Q_{\theta}^{\mu_i,\bar{\gamma}}(s, a)\}_{i=1}^K$ denote the value estimations over the specific policies, we train the value function by minimizing the following objective:
\begin{align}
    &\quad~ J(\theta):=\dfrac{1}{2}\mathbb{E}\left[\left(Q_{\theta}^{\pi,\gamma}(s,a) - \mathcal{T}^{\pi}_{\gamma}(s',r)\right)^2 + \sum_{i=1}^K \left(Q_{\theta}^{\mu_i,\bar{\gamma}}(s,a) - \mathcal{T}^{\mu_i}_{\bar{\gamma}}(s',r)\right)^2\right],\label{eq:value_loss_hybrid}\\
    \text{where}~~&\mathcal{T}^{\pi}_{\gamma}(s',r) := r + \gamma~\mathbb{E}_{a'\sim \pi(\cdot|s')}\left[ Q_{\bar{\theta}}^{\pi,\gamma}(s',a')\right],\nonumber \qquad\qquad~~~\text{(long-term task policy operator)}\\
    \qquad\quad~~&\mathcal{T}^{\mu_i}_{\bar{\gamma}}(s',r) := r + \bar{\gamma}~\mathbb{E}_{a'\sim \mu_i(\cdot|s')}\left[ Q_{\bar{\theta}}^{\mu_i,\bar{\gamma}}(s',a')\right],\nonumber \qquad\quad~\text{(short-term prior policy operator)}
\end{align}
where $\bar{\theta}$ denotes the parameters of the target network. With different horizons $\gamma$ and $\bar{\gamma}$, the value estimations on task policy $\pi$ and prior policies $\{\mu_i\}_{i=1}^K$ consider long and short-term behaviors, respectively. 
We achieve computational efficiency and scalability through the shared architecture, which is crucial for handling abundant prior policies.

\begin{figure}
    \centering
    \includegraphics[width=\textwidth]{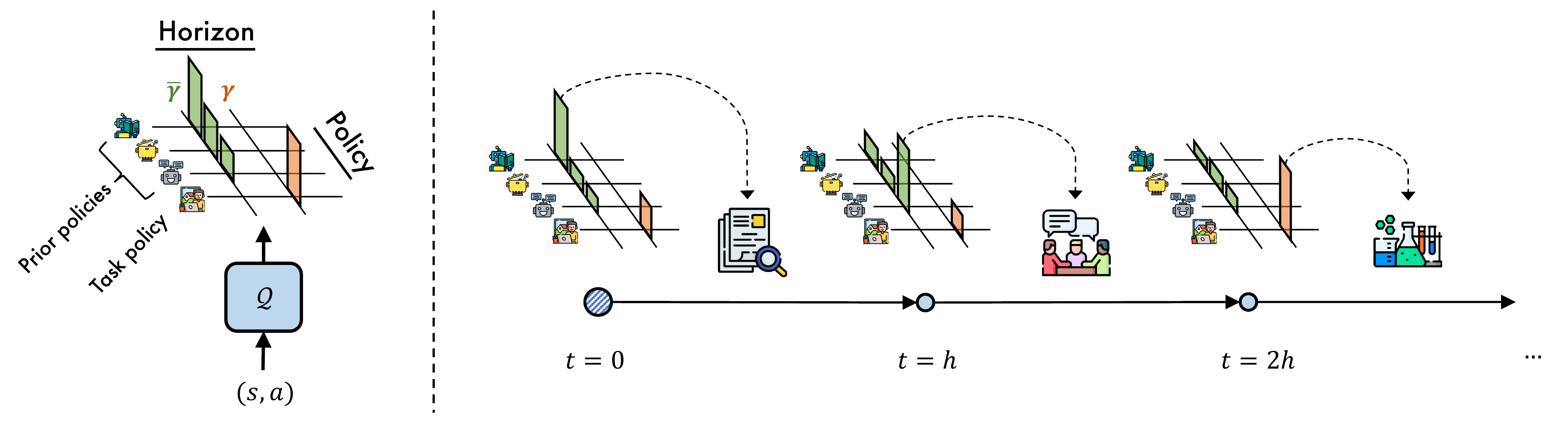}
    \caption{\textit{(Left)}~The architecture of the hybrid value function that estimates the values over all policies with different horizons.~\textit{(Right)}~The semantic illustration of value-guided behavior planning.}
    \vspace{-1em}
    \label{fig:method_semantic}
\end{figure}

\subsection{Value-guided Behavior Planning}
\label{sec:method:exploit}
We perform value-guided behavior planning every $h$ step to exploit prior policies' short-term behaviors, as shown in Figure~\ref{fig:method_semantic}~(Right).
Specifically, we compare the value estimations of all policies, namely $\{Q_{\theta}^{\pi,\gamma}(s, a)\}\cup\{Q_{\theta}^{\mu_i,\bar{\gamma}}(s, a)\}_{i=1}^K$, and choose the policy that yields the highest value estimation for the subsequent $h$-step interactions, which results in a behavior policy $\eta$ that can be formulated as:
\begin{equation}
    \eta(\cdot|s_t) := \underset{\nu\in\{\pi\}\cup\{\mu_i\}_{i=1}^K}{\arg\max} Q^\nu_{\theta}(s_{t^-},a_{t^-})~(\cdot|s_t),
    \label{eq:q_guided_behavior_policy}
\end{equation}
where $t^-:= \lfloor\frac{t}{h}\rfloor\cdot h$ denotes the time step of the last switch point, and the discount factor superscripts are omitted for simplicity. The policy switch only occurs every $h$ step, and the selected policy governs the interactions for the next $h$ steps.

Intuitively, short-term behavior estimation might induce myopic behaviors that hinder the agent from collecting long-term optimal transitions. However, such utilization of prior policies can still facilitate learning through potentially sharable short-term behaviors. At the early training stage, the immature task policy almost induces low behavior evaluation $Q^{\pi,\gamma}_{\theta}$. In contrast, prior policies with semantically meaningful behaviors can provide short-term beneficial behaviors. By performing value-guided behavior planning, we greedily conduct the most promising behaviors based on the value estimations. As the training proceeds, the performance of the task policy $\pi$ improves, which induces higher value estimations $Q^{\pi,\gamma}_\theta$. Thus, the value-guided behavior planning weans off the prior policies automatically when the performance of the task policy improves, eliminating the negative impact of the prior policies at the later training stage. 

In order to further leverage the behaviors of prior policies for diverse experience collection, it is essential to try different policies for complex, temporally-extended behaviors.
To achieve this, we introduce a heuristic method inspired by upper confidence bound~(UCB)~\cite{auer2002using}. Specifically, we combine the value estimations and the policy selection counts to perform the behavior planning formulated as follows:
\begin{equation}
    \tilde{\eta}(\cdot|s_t) := \underset{\nu\in\{\pi\}\cup\{\mu_i\}_{i=1}^K}{\arg\max} \left[ Q^\nu_{\theta}(s_{t^-},a_{t^-}) + c\cdot \sqrt{\dfrac{\log(2T)}{N_\nu + N_{\nu^-\rightarrow\nu}}}\right]~(\cdot|s_t),
    \label{eq:q_ucb_guided_behavior_policy}
\end{equation}
where $t^-:= \lfloor\frac{t}{h}\rfloor\cdot h$ denotes the time step of the last switch point, $c$ denotes the trade-off coefficient that can be regarded as a hyperparameter, $T$ denotes the total policy selection counts, $N_\nu$ denotes the counts of selecting policy $\nu$, $N_{\nu^-\rightarrow \nu}$ denotes the counts of the transformation from the last selected policy $\nu^-$ to policy $\nu$. We introduce the transformation counts $N_{\nu^-\rightarrow \nu}$ to obtain diverse behavior patterns by encouraging various policy combinations at the switch points.

\subsection{Theoretical Analysis}
\label{sec:method:theory}
This section analyzes the behavior policy $\eta$ induced by value-guided behavior planning. Since the behaviors of the prior policies can be sub-optimal or even harmful in the current task, deploying the prior policies might result in worse performance than only using the task policy. Thus, we aim to provide a performance guarantee on the behavior policy $\eta$ induced by Eq.~(\ref{eq:q_guided_behavior_policy}).

\begin{theorem}
    \label{theorem:value_bound}
    Following the behavior policy induced by Eq.~(\ref{eq:q_guided_behavior_policy}), when the prior policy $\bar{\mu}:= \arg\max_{\mu\in\{\mu_i\}_{i=1}^K}V^\mu_{\bar{\gamma}}(s_j)$ meets $V^{\bar{\mu}}_{\bar{\gamma}}(s_j)\geq V^\pi_{\gamma}(s_j), \forall s_j \in S, \exists~j \in [0, h, 2h, \dots]$, $\bar{\gamma}<\gamma$, and the policy $\eta$ is fixed after the switch. the induced value of $\eta$ can be bounded as follows:
    \begin{equation*}
        V_\gamma^{\eta}(s_j) - V_\gamma^{\pi}(s_j) ~\geq~ \dfrac{\gamma-\bar{\gamma}}{(1-\gamma)(1-\bar{\gamma})}R_{\text{max}}>0.
    \end{equation*}
\end{theorem}
Proof is in Appendix~\ref{app:theory:performance},~Theorem~\ref{theorem_app:value_bound}. The result above reveals that, by using the prior policy $\bar{\mu}$ for the remaining interactions~(\textit{i.e.},~$t>j$) via the value guidance, the behavior policy enjoys a performance guarantee compared with the performance induced by simply using task policy $\pi$. Based on the result, we further provide the performance guarantee on the behavior policy in the case of a single switched sub-trajectory.

\begin{theorem}
    \label{thm:perf_bound}
    When there is only one sub-trajectory from $kt$ to $(k+1)h$ during which a prior policy $\bar{\mu}$ is selected, which means no prior policy $\mu \in \{\mu_i\}_{i=1}^K$ satisfies $V_{\bar{\gamma}}^\mu(s_t) \geq V_\gamma^\pi(s_t)$ except $t \in [kh, (k+1)h)$. The performance difference between the behavior policy $\eta$ induced by Eq.~(\ref{eq:q_guided_behavior_policy}) and the task policy $\pi$ is bounded as follows:
    $$
    J_\gamma(\eta) - J_\gamma(\pi) \geq \gamma^{kh} \dfrac{\gamma-\bar{\gamma}}{(1-\gamma)(1-\bar{\gamma})}R_{\text{max}} - \gamma^{(k+1)h} \dfrac{R_{\text{max}}}{(1-\gamma)^2}\|\bar{\mu}-\pi\|_\infty,
    $$
    where $\|\bar{\mu}-\pi\|_\infty := \underset{s\in S}{\sup}\underset{A}{\sum}|\bar{\mu}(a|s) - \pi(a|s)|$, and $\bar{\mu} = \arg\max_{\mu\in\{\mu_i\}_{i=1}^K} V^\mu_{\bar{\gamma}} (s_{kh}),~\forall s_{kh}\in \{s\in S | \mathbb{P}_{kh}^\pi(s) > 0\}$.
\end{theorem}
Proof is in Appendix~\ref{app:theory:performance},~Theorem~\ref{thm_app:perf_bound}. The result above implies that the performance of the behavior policy, inducing a single sub-trajectory during which the prior policy $\bar{\mu}$ takes control, is guaranteed to be not very different from the performance of task policy $\pi$. While the policy selection dynamic in practice can be more complicated than a single switch sub-trajectory, our theoretical results imply that value-guided behavior planning enjoys rigorous performance guarantees in non-trivial cases.

\section{Related Work}
\label{sec:rw}

\textbf{Policy reuse.}~
Previous works have examined single-task policy reuse~\cite{laroche2017multi,gimelfarb2018reinforcement,kurenkov2020ac,cheng2020policy,uchendu2022jump,agarwal2022reincarnating,zhangpolicy} and cross-task policy reuse~\cite{fernandez2006probabilistic,rosman2016bayesian,li2018optimal,li2019context,yang2020efficient,kurenkov2020ac,gimelfarb2021contextual,vezzani2022skills,zhangcup,campos2021beyond,zhang2023efficient}. Given a collection of prior policies without knowing their properties, the approaches designed for cross-task policy reuse can generalize to the single-task setting, and we focus on the cross-task setting in this work. Existing approaches can be categorized into three classes: advantage-based methods~\cite{cheng2020policy,zhangcup}, aggregation-based methods~\cite{barekatain2021multipolar,gimelfarb2021contextual}, and behavior-based methods~\cite{fernandez2006probabilistic,rosman2016bayesian,li2018optimal,li2019context,yang2020efficient,vezzani2022skills,campos2021beyond,zhang2023efficient,kurenkov2020ac}. The advantage-based methods perform policy regularization with the superior actions advised by all prior policies~\cite{cheng2020policy,zhangcup}. Aggregation-based methods attempt to compose all actions via certain aggregation functions and to learn the task policy by optimizing the aggregation functions~\cite{barekatain2021multipolar,gimelfarb2021contextual}. In contrast, behavior-based methods directly deploy the prior policies for guided online interactions~\cite{fernandez2006probabilistic,rosman2016bayesian,li2018optimal,li2019context,yang2020efficient,kurenkov2020ac,vezzani2022skills,campos2021beyond,zhang2023efficient}, which simultaneously exploits the advantageous actions and is computationally efficient given abundant prior policies. Some behavior-based methods select the policies via the simple heuristics~(\textit{e.g.}~$\epsilon$-greedy)~\cite{fernandez2006probabilistic,li2018optimal}, which lacks a principled mechanism to evaluate the policies for subsequent interactions. To properly guide the policy selection, several works formulate the problem with hierarchical control~\cite{li2019context,yang2020efficient,vezzani2022skills}, which results in nonstationary training and requires extensive rollouts of each low-level policy. Furthermore, the value function is adopted for policy selection in several works~\cite{campos2021beyond,zhang2023efficient,kurenkov2020ac}. Though effective, the estimated value function can only provide biased behavior evaluation since the value function is typically trained concerning the task policy. Our method falls into the behavior-based method category and utilizes value functions to deploy the short-term behaviors of prior policies. In contrast to existing works, we propose a hybrid value architecture to perform decomposed evaluations with different horizons, maintaining consistency between the behavior evaluation and the behavior deployment.

\textbf{Composing low-level primitives or skills.}~ Many works have investigated composing low-level policies that can be unsupervised learned skills~\cite{eysenbach2018diversity,merel2018neural,strouse2022learning,ajay2021opal,gregor2016variational,Sharma2020Dynamics-Aware,yang2023behavior} or preexisting action primitives~\cite{Qureshi2020Composing,dalal2021accelerating,peng2019mcp}. The essential difference between these works and those in policy reuse is the assumption on the prior policies or the skills. Concretely, the policy reuse setting does not assume specific properties and accessible knowledge of the prior policies, while they can only be queried like multiple black-box functions. In contrast, the skills are typically discovered by maximizing the state coverage~\cite{eysenbach2018diversity,strouse2022learning} or controllability~\cite{gregor2016variational,Sharma2020Dynamics-Aware} over the environment. Another line of work learns the skills by extracting semantic behaviors from substantial demonstrations or offline datasets~\cite{merel2018neural,ajay2021opal}. Furthermore, primitive-based works typically assume universal behavior abstractions that can be composed into a wide range of behaviors~\cite{Qureshi2020Composing,dalal2021accelerating}. Therefore, our work complements existing skill-composing frameworks and can be plugged into any skill discovery methods for downstream adaptation. For example, the skills learned via various skill discovery methods can be regarded as individual prior policies and be reused via our algorithm in the downstream tasks for efficient learning. 

\textbf{Value estimation with different horizons.}~The discount factor that exponentially reduces the present value of future rewards~\cite{bellman1957markovian,sutton2018reinforcement} establishes a time preference for rewards realized sooner or later. Previous works have investigated the role of a lower discount factor on the approximation error bound~\cite{petrik2008biasing}, model accuracy~\cite{jiang2015dependence}, regularization~\cite{amit2020discount,chen2018improving}, and the pessimistic effect in offline setting~\cite{hu2022role}. Furthermore, empirical improvement has been observed in previous methods that incorporate multiple horizons~\cite{reinke2017average,romoff2019separating,sherstan2020gamma,tang2021taylor,kalweit2019composite} or an adaptive horizon~\cite{xu2018meta} for optimization. Orthogonal to these works, we propose to evaluate the behaviors with different horizons across multiple policies for policy reuse. 

\textbf{Hybrid architectures of value functions.}~
Existing works have explored different value functions rather than estimating the expected value with a classical architecture~\cite{watkins1992q,rummery1994line}. Many works have investigated the distributional value functions that estimate the full distribution of the returns~\cite{bellemare2017distributional,dabney2018distributional,dabney2018implicit,bai2022monotonic}. The universal value function approximator~(UVFA)~\cite{schaul2015universal} has been proposed to model the goal-directed knowledge~\cite{andrychowicz2017hindsight,nasiriany2019planning}. Furthermore, hybrid value functions have been proposed for separate evaluation over decomposed reward functions~\cite{van2017hybrid,lin2020rd,zhang2021distributional} or different time-scales~\cite{sherstan2020gamma}. In contrast to these works, we propose a novel hybrid value function that evaluates different policies. The works manipulating successor features~\cite{barreto2018transfer} evaluate multiple policies as well, but they assume the value functions of different policies are linearly combinable over the shared feature~\cite{barreto2018transfer,borsa2018universal}.

\section{Experiments}
\label{sec:exp}

In this section, we empirically evaluate SMEC to answer the following questions: (1)~Can SMEC improve training efficiency given a collection of prior policies obtained from different tasks?~(Figure~\ref{fig:curves_aggregated}) (2)~How do the algorithm designs and hyper-parameters affect the performance?~(Figure~\ref{fig:effect_factor_multiple},~\ref{fig:effect_hyper_multiple}) (3)~Can SMEC identify the prior policies and fully exploit the beneficial ones that are related to the current task?~(Figure~\ref{fig:prior_utilization},~\ref{fig:more_priors}) (4)~Can SMEC select the prior policies properly at the early training stages and automatically wean off the prior policies as the task policy improves?~(Figure~\ref{fig:policy_switch_dynamic_metaworld_pickplacewall}) 

\subsection{Setups}
\label{sec:exp:setup}

\textbf{Training environment settings.}~To rigorously validate the effect of our method, we use a manipulation benchmark~(\textit{i.e.},~MetaWorld~\cite{yu2020meta}) and a locomotion environment~(\textit{i.e.},~AntMaze~\cite{fu2020d4rl}) that can provide diverse task instances for the evaluation. We use 3 prior policies~(learned policies in \textit{Push}, \textit{Reach}, and \textit{PickPlace}) for Meta-World and 4 prior policies~(learned policies for reaching four goals in the empty maze) for AntMaze. For the downstream tasks, we choose 12 tasks in MetaWorld different from those of prior policies and 3 tasks in AntMaze with more complex maze layouts. Details of the environments are shown in Appendix~\ref{app:setting:env}, including the performance of the prior policies on each downstream task.

\textbf{Implementation and baselines.}~We use SAC~\cite{haarnoja2018soft} as our backbone algorithm. The effective horizon of the short-term behaviors $h$ is set to one-tenth of the episode length $H$~(\textit{i.e.},~$h:=H/10$) across all environments. We compare our method with the following baselines: (\romannumeral1)~\textit{Scratch}: training SAC from scratch without access to any prior policy; (\romannumeral2)~\textit{AC-Teach}~\cite{kurenkov2020ac}: a bayesian method leveraging behaviors of prior policies; (\romannumeral3)~\textit{CUP}~\cite{zhangcup}: an advantage-based method performing policy regularization with actions advised by prior policies; (\romannumeral4)~\textit{MultiPolar}~\cite{barekatain2021multipolar}: an aggregation-based method composing actions via the auxiliary network; (\romannumeral5)~\textit{MAMBA}~\cite{cheng2020policy}: a method performing policy improvement with a baseline function aggregated over all value functions; (\romannumeral6)~\textit{SkillS}~\cite{vezzani2022skills}: a hierarchy-based method performing policy sequencing for temporally-extended exploration; (\romannumeral7)~\textit{QMP}~\cite{zhang2023efficient}: a behavior sharing method exploiting actions from prior policies. The details of the implementation and baselines are in Appendix~\ref{app:setting:alg}.

\begin{figure}[t]
    \centering
    \includegraphics[width=\textwidth]{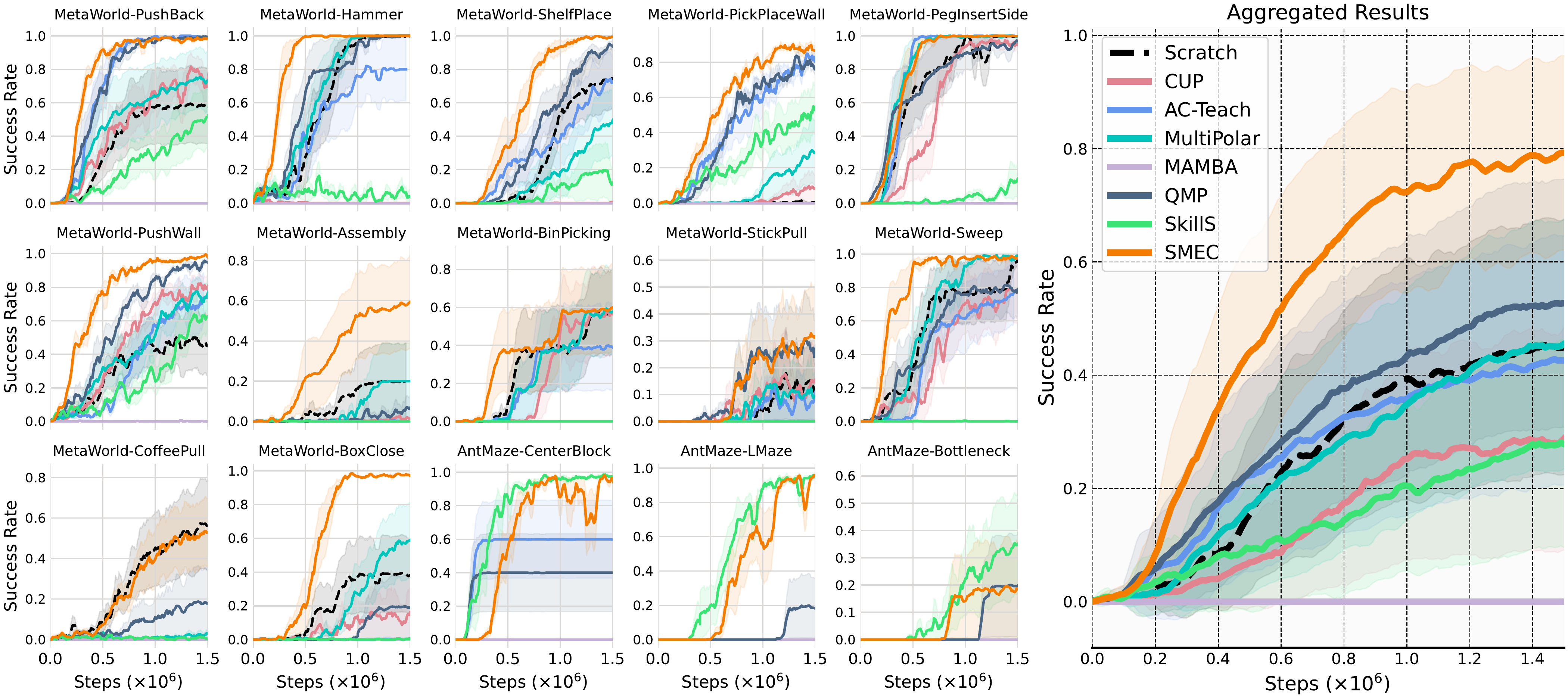}
    \caption{(\textit{Left})~The learning curves on the success rate across all tasks, including $12$ MetaWorld tasks and $3$ AntMaze tasks. The solid line and shaded regions represent the mean and standard deviation across five runs with different random seeds. (\textit{Right})~The aggregated curves over all 15 tasks.}
    \label{fig:curves_aggregated}
    \vspace{-1em}
\end{figure}

\subsection{Sample-efficient Training by Reusing Prior Policies}
\label{sec:exp:basic_results}
To examine the performance of our method in exploiting the prior policies for sample-efficient training, we compare our method with the baselines given the same prior policies mentioned in Section~\ref{sec:exp:setup}. The results in Figure~\ref{fig:curves_aggregated} demonstrate that SMEC consistently improves the performance of \textit{Scratch} across all environments, which validates the ability of SMEC to exploit prior policies. We remark that the prior policies can be useless in the downstream task, in which case improperly using the prior policies can hinder the training of task policy, resulting in inefficient learning. However, SMEC outperforms the baselines in all environments, which validates the effectiveness of utilizing short-term behaviors guided by the decomposed value estimations.

Specifically, we observe that the advantage-based methods~(\textit{CUP},~\textit{MAMBA}) consistently underperform the other algorithms in most tasks, which results from the insufficient exploration induced by the policy regularization in the early training stage. The aggregation-based method~(\textit{MultiPolar}) underperforms \textit{Scratch} in several tasks, as learning to aggregate irrelevant actions from prior policies can hinder the training. Furthermore, prior value-guided behavior-based methods~(\textit{AC-Teach},~\textit{QMP}) are less efficient than our method, which can result from the inconsistency between the evaluation of the prior policies and the deployed behaviors. The hierarchical method~(\textit{SkillS}) is inefficient in all MetaWorld tasks while performing well in the AntMaze tasks. Since the near-optimal trajectories in the AntMaze tasks can be obtained by composing the prior policies, \textit{SkillS} can efficiently learn the optimal high-level controller. In contrast, the prior policies in MetaWorld are nearly useless in several tasks, and the abundant deployments of prior policies will hinder the training. SMEC is the only method that outperforms \textit{Scratch} across all environments. Benefiting from value-guided behavior planning, the prior policies would only be deployed if the short-term behaviors of the prior policy are better than those of the task policy, which leads to the cautious deployment of the prior policies.

\subsection{Ablation Studies}
\label{sec:exp:ablation}
In this subsection, we conduct ablation experiments to analyze the effect of several factors on the performance and the sensitivity to the hyper-parameters. For each set of experiments, we report the aggregated results on multiple tasks with 5 different runs for each task. The detailed results on each task are deferred to Appendix~\ref{app:addition:ablation}, and additional results are deferred to Appendix~\ref{app:addition} and \ref{app:continual}.

\textbf{Evaluation of prior policies.}~
To evaluate the effect of the disentangled evaluation for prior policies proposed in Section~\ref{sec:method:evaluation}, we introduce a variant that performs policy selection via single behavior value function $Q_\theta$~(\textit{i.e.},~$\eta(\cdot|s_t):=\arg\max_{\nu\in\{\pi\}\cup\{\mu_i\}_{i=1}^K}\mathbb{E}_{a\sim\nu(\cdot|s_{\lfloor\frac{t}{h}\rfloor\cdot h})}[Q_\theta(s_{\lfloor\frac{t}{h}\rfloor\cdot h}, a)]~(\cdot|s_t)$) and temporally-extended exploration same as SMEC~(\textit{i.e.}, select a policy every $h$ step). The value function of the variant is trained with the task policy value targets the same as standard off-policy algorithms. We compare the performance of the original algorithm with the variant on multiple environments, and the aggregated results are shown in Figure~\ref{fig:effect_factor_multiple}~(Left). The results indicate that the decoupled evaluation of prior policies helps with performance, which can result from improved policy selection based on the accurate evaluation of the prior policy behaviors.

\begin{figure}[t]
    \centering
    \begin{minipage}[h]{.58\textwidth}
        \centering
        \includegraphics[width=\textwidth]{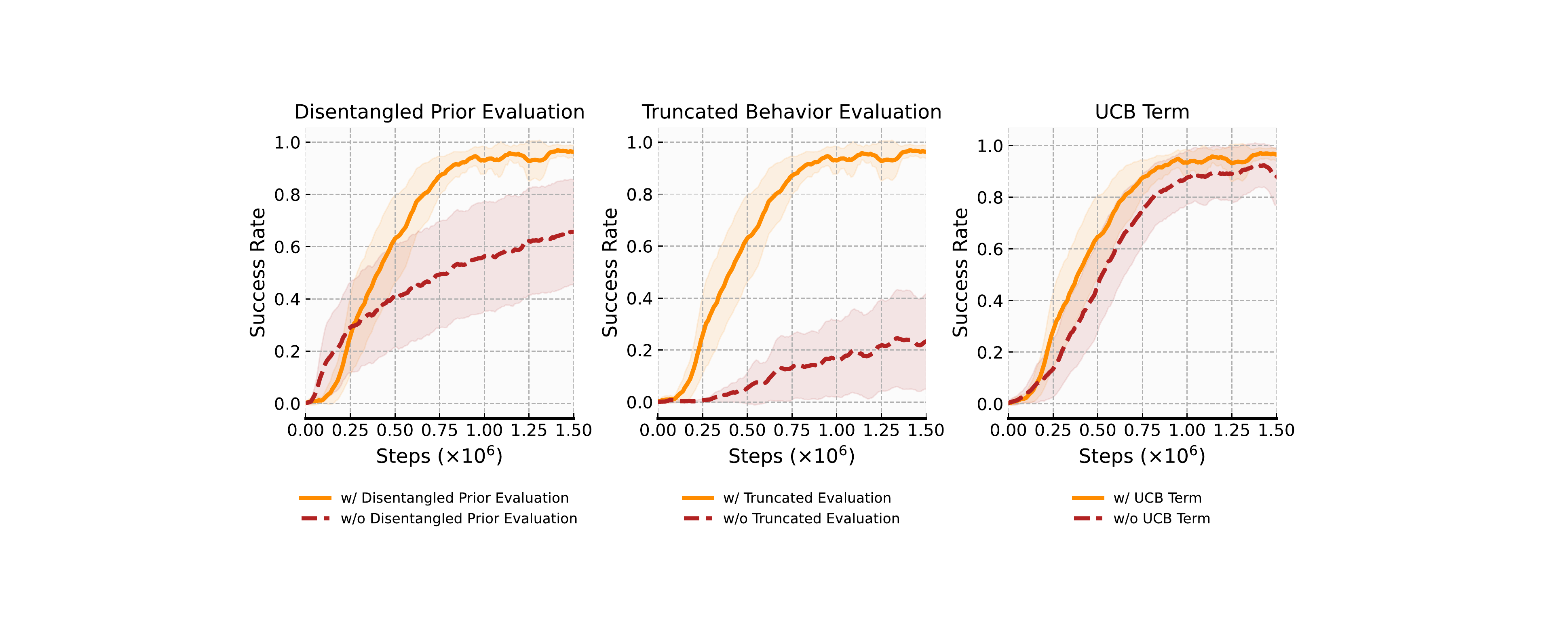}
        \captionof{figure}{Ablation results on algorithm factors. (\textit{Left})~Effect of disentangled evaluation of prior policies. (\textit{Middle})~Effect of truncated behavior evaluation. (\textit{Right})~Effect of UCB policy selection term.}
        \label{fig:effect_factor_multiple}
    \end{minipage}
    \hfill
    \begin{minipage}[h]{.395\textwidth}
        \centering
        \includegraphics[width=\textwidth]{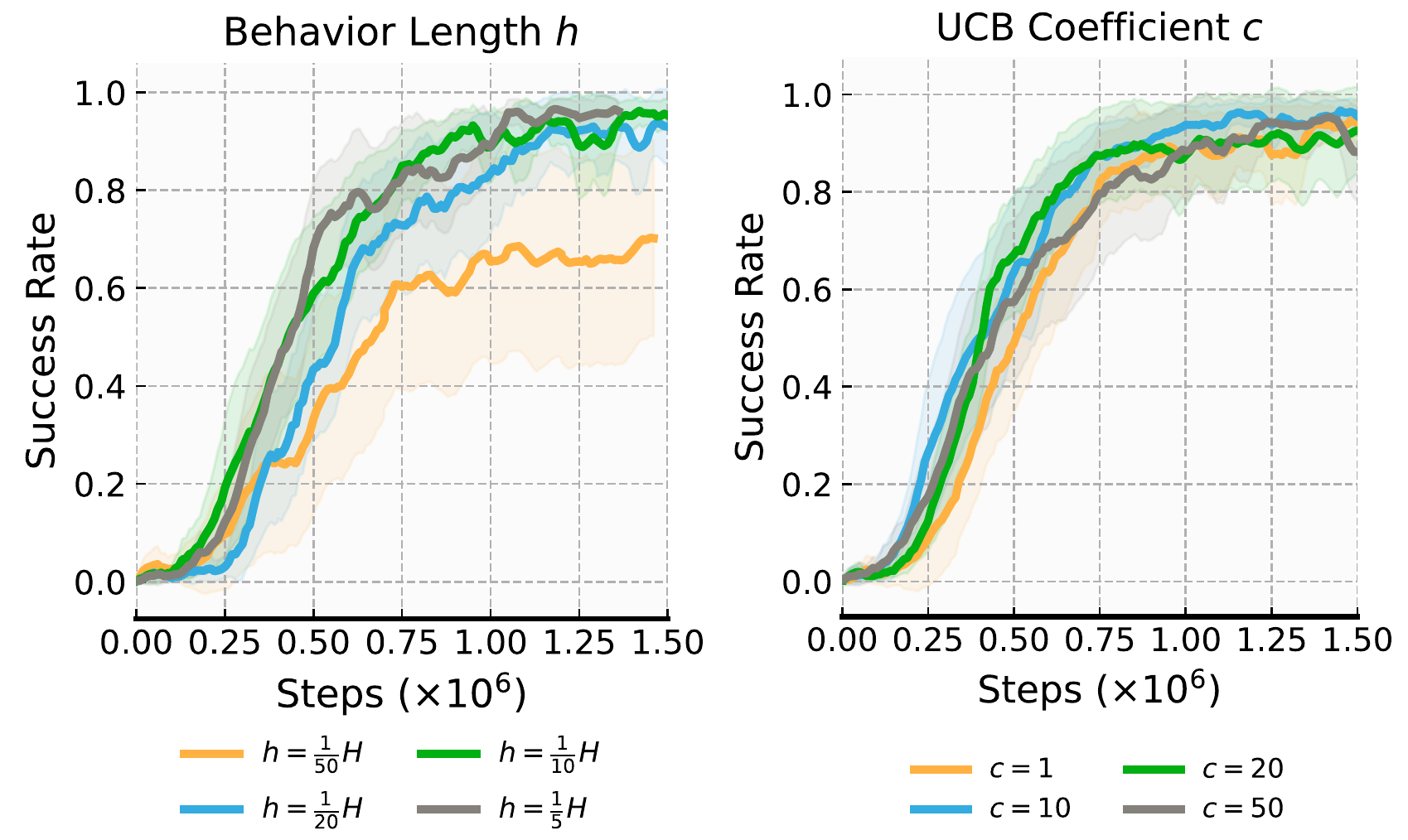}
        \captionof{figure}{Ablation results on hyper-parameters. (\textit{Left})~Sensitivity to the exploration length $h$. (\textit{Right})~Sensitivity to the UCB coefficient $c$.}
        \label{fig:effect_hyper_multiple}
    \end{minipage}
    \vspace{-1em}
\end{figure}

\textbf{Short horizon evaluation.}~
We proposed to validate the effect of truncated behavior evaluation of the prior policies via the lower discount factor $\bar{\gamma}$. To achieve this, we propose a variant that trains the value function of prior policies with the same horizon as the task policy~(\textit{i.e.},~$\bar{\gamma}=\gamma$). The aggregated results on multiple tasks are shown in Figure~\ref{fig:effect_factor_multiple}~(Middle), which validate the importance of the short-horizon evaluation. Interestingly, the performance degradation induced by removing truncated behavior evaluation is more significant than that induced by removing disentangled evaluation. We defer further discussion on this phenomenon to Appendix~\ref{app:addition:discussion_normal_gamma}.

\textbf{UCB-term.}~
We perform ablation experiments on the UCB component presented in Section~\ref{sec:method:exploit}. We compare the original method with the variant performing simply value-guided policy selection in Eq.~(\ref{eq:q_guided_behavior_policy}). The results are shown in Figure~\ref{fig:effect_factor_multiple}~(Right), which validates the incremental performance improvement by inclusion UCB term. However, by referring to the detailed results in Appendix~\ref{app:addition:ablation}, we observe that the UCB term is essential for significant performance improvement in a few tasks. Furthermore,
we compare variants with different coefficient values to examine the impact of UCB coefficient~(\textit{i.e.},~$c$). The results in Figure~\ref{fig:effect_hyper_multiple}~(Right) demonstrate that different coefficient values do not significantly influence our method. Since the Q value of the task policy increase as the training proceeds, UCB-based policy selection would gradually prefer the task policy. Thus, different coefficients only influence the training dynamic at the very early training stage.

\textbf{Sensitivity to behavior length $h$.}~
We further investigate the role of the behavior length $h$ on the algorithm performance. The results shown in Figure~\ref{fig:effect_hyper_multiple}~(Left) demonstrate that the performance of our method only degrades when the behavior length is overly short~(\textit{e.g.},~$h=\frac{1}{50}H$), which can result from the insufficient time budget for effective temporally-extended behaviors.



\subsection{Identification of Beneficial Prior Policies}
\label{sec:exp:identify}
Intuitively, the utilization of the prior policies should align with the effectiveness of the prior policies in the current task. Thus, we aim to validate whether SMEC can identify the beneficial prior policies and maximally exploit them in this subsection. 

\begin{figure}[h]
    \vspace{-0.5em}
    \centering
    \begin{minipage}[t]{.44\textwidth}
        \centering
        \includegraphics[width=\textwidth]{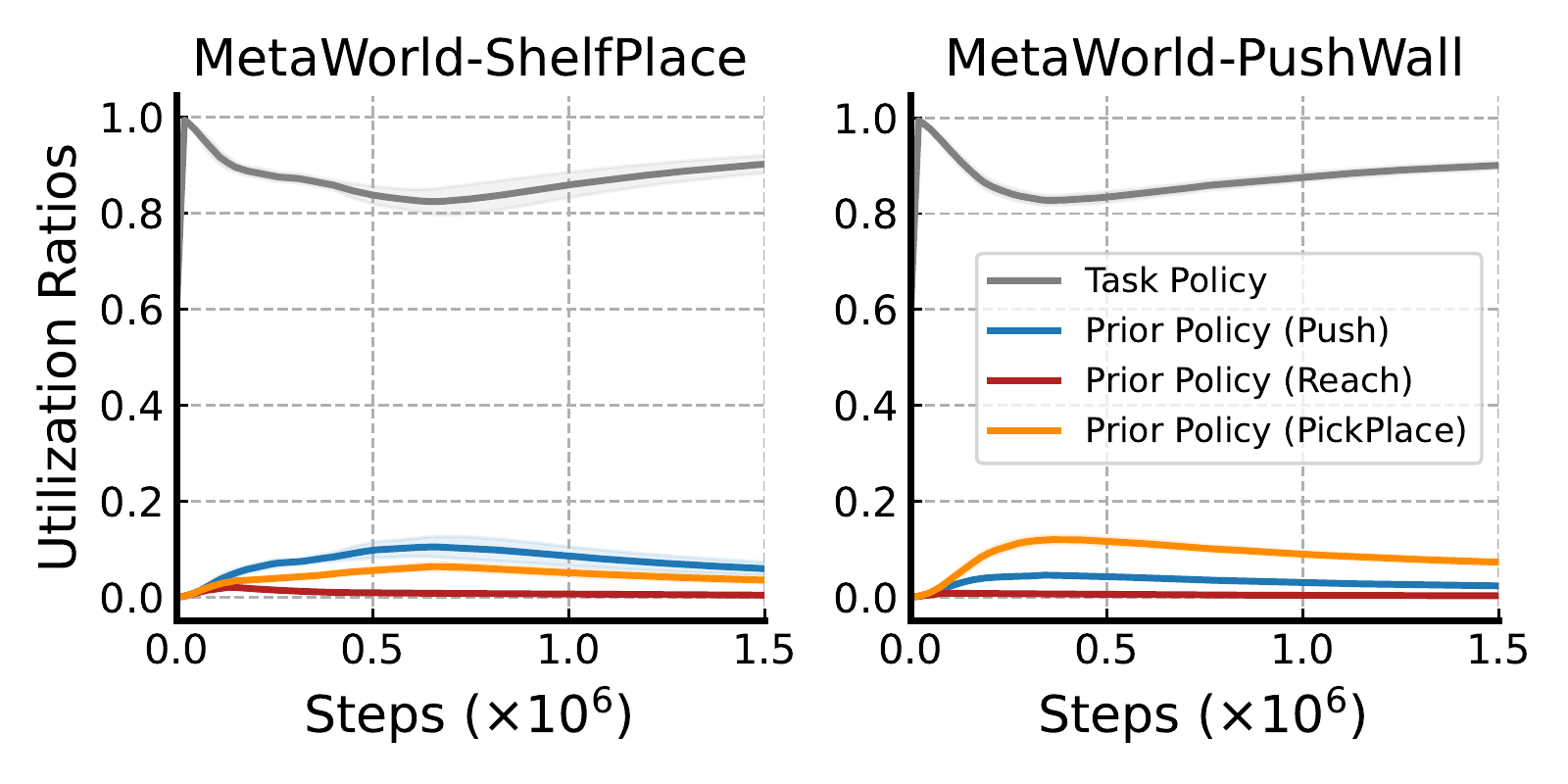}
        \captionof{figure}{Utilization ratios of all policies in two MetaWorld tasks.}
        \label{fig:prior_utilization}
    \end{minipage}
    \hfill
    \begin{minipage}[t]{.52\textwidth}
        \centering
        \includegraphics[width=\textwidth]{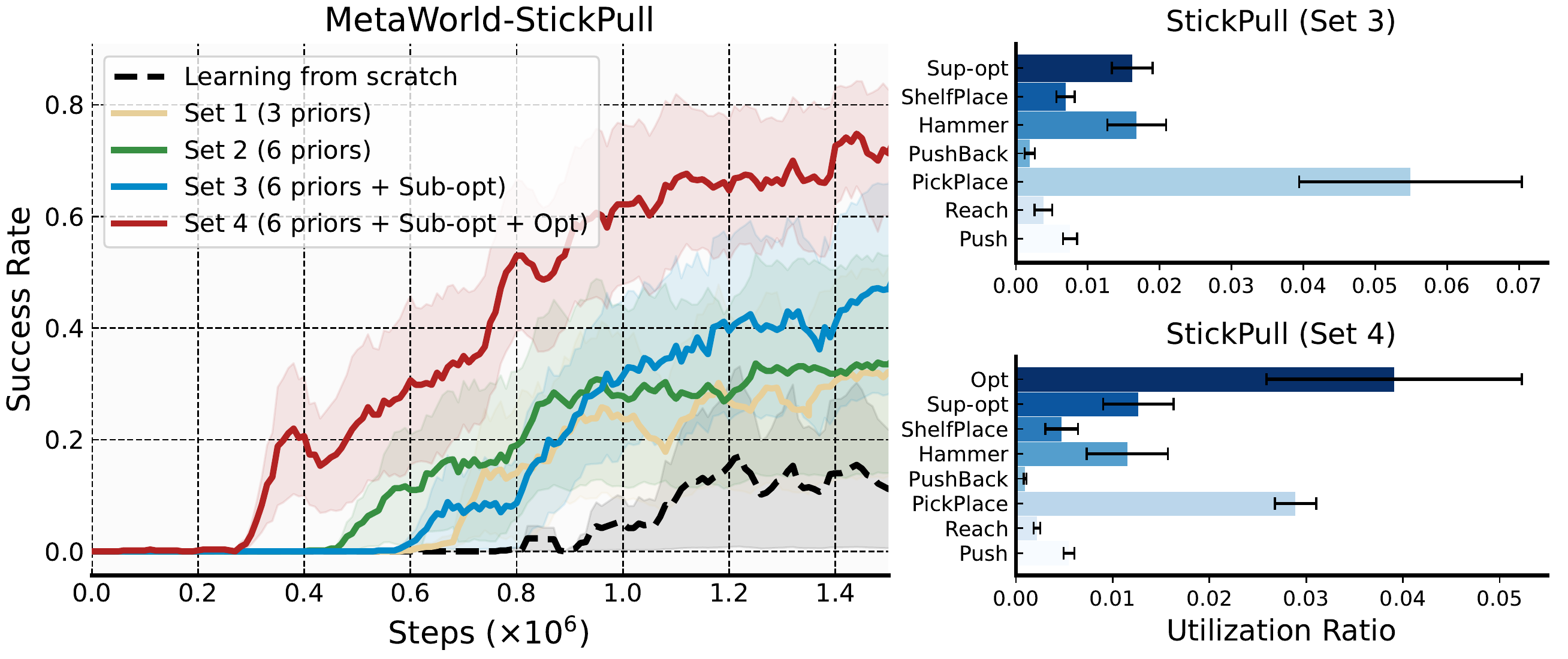}
        \captionof{figure}{(\textit{Left})~Performances under different prior policies. (\textit{Right})~The utilization ratios of prior policies in different settings.}
    \label{fig:more_priors}
    \end{minipage}
    \vspace{-0.5em}
\end{figure}

\textbf{Utilization of prior policies.}~We first analyze the utilization of the prior policies induced by SMEC. Based on the experiments in Section~\ref{sec:exp:basic_results}, Figure~\ref{fig:prior_utilization} shows the percentages of each policy selected throughout training on ShelfPlace and PushWall of MetaWorld. In ShelfPlace, the Push policy is chosen more frequently than the other two prior policies. In PushWall, the PickPlace policy is the most selected prior policy, which seems to contradict the task-policy correlation. However, as demonstrated in Figure~\ref{fig:zero-shot_performance_mix} of Appendix~\ref{app:setting:env}, the PickPlace policy is even more practical than the Push policy in the PushWall task. Thus, SMEC can efficiently identify the effectiveness of the prior policies. At the early training stage, the utilization of prior policies increases as the value estimation on prior policies gradually becomes accurate. As the training proceeds, the task policy becomes better, leading to an increase in its utilization. Complete results of the prior policy utilization across all tasks are shown in Appendix~\ref{app:addition:full_utilization}.

\textbf{Learning with increasing prior policies.}~
We further conduct experiments with different sets of prior policies. On the task StickPull, we use four different prior policy sets: Set 1~(Reach, Push, PickPlace policy), Set 2~(additional PushBack, Hammer, ShelfPlace policy), Set 3~(additional task-specific sub-optimal policy), Set 4~(additional task-specific optimal policy). The results in Figure~\ref{fig:more_priors} show that the performance of SMEC improves as the prior policies increase, especially in the cases task-specific policies exist in the prior policy set. By investigating the utilization ratios of prior policies shown in Figure~\ref{fig:more_priors}~(Right), we observe that the task-specific optimal policy is the most selected one given the Set 4 prior policies, which validates that SMEC can identify the related policies and maximally exploit them. Extra results with different prior policy sets on CoffeePull are deferred to Appendix~\ref{app:addition:more_priors}.

\subsection{Qualitative Analysis}
\label{sec:exp:qualitative}

In this section, we provide visualization results on the dynamic of policy selection induced by SMEC, hoping to verify the abilities to exploit the prior policies and to wean off the prior policies as the performance of the task policy improves. 
The visualization of the policy switch dynamic throughout training shown in Figure~\ref{fig:policy_switch_dynamic_metaworld_pickplacewall} demonstrates that our method can exploit the related prior policies at the early training stage and the prior policies are gradually weaned off as the performance of task policy improves. 

\begin{figure}[h]
    \centering
    \vspace{-0.25em}
    \includegraphics[width=0.925\textwidth]{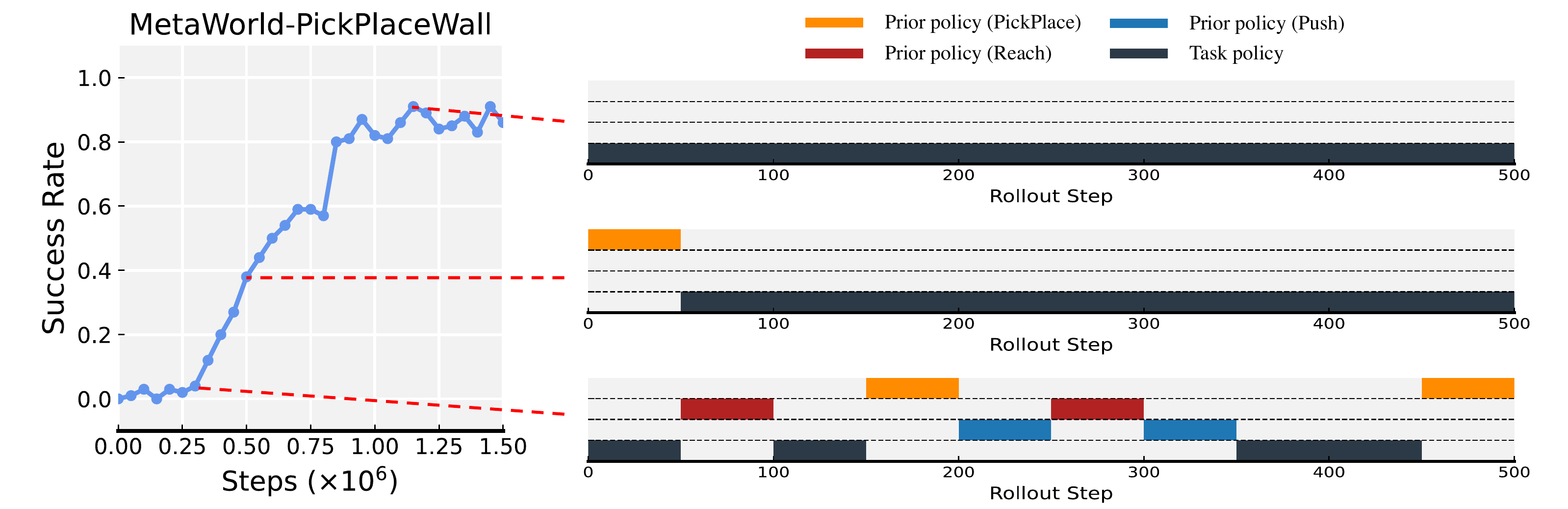}
    \caption{(\textit{Left})~The performance of the task policy throughout training.  (\textit{Right})~The policy switch dynamic within an episode at different training stages. The shaded areas indicate the time when the corresponding policy takes control.}
    \label{fig:policy_switch_dynamic_metaworld_pickplacewall}
    \vspace{-1em}
\end{figure}

\section{Conclusion}
\label{sec:conclusion}

This paper introduces Selective Myopic bEhavior Control~(SMEC), a simple yet effective approach for policy reuse. We start with the insight that the short-term behaviors of prior policies can be sharable for effective policy reuse. To achieve this, we propose to evaluate the short-term behaviors of the prior policies and perform behavior planning based on the value estimations across all policies. Based on the proposed hybrid value architecture, SMEC can efficiently scale to many prior policies. Theoretically, we analyze the performance of the behavior policy induced by SMEC and demonstrate that the performance is guaranteed via the proposed value-guided behavior planning. Empirically, we show that our method outperforms various baseline methods across manipulation and locomotion domains. We validate the abilities of SMEC to identify the relevant prior policies and to automatically wean off the prior policies as the performance of task policy improves. 

\textbf{Limitation and future directions.}
Though the proposed method demonstrates advanced performance, the learning efficiency on novel tasks is limited by the utility of the prior policies. To solve the challenging tasks with ineffective prior policies, the skill discovery schemes~\cite{eysenbach2018diversity,ajay2021opal} that learn multiple skills in online~\cite{eysenbach2018diversity} or offline~\cite{ajay2021opal} setting can be integrated to pre-train diverse prior policies. Furthermore, existing policy reuse mainly focus on enhancing the learning efficiency in downstream tasks. It will be interesting to extend the policy reuse for safe exploration~\cite{bharadhwaj2021conservative} or risk-sensitive decisions~\cite{chow2015risk} by utilizing the relevant behaviors of the prior policies.


\bibliographystyle{plain}

\bibliography{ref}

\begin{thebibliography}{10}

\bibitem{achiam2017constrained}
Joshua Achiam, David Held, Aviv Tamar, and Pieter Abbeel.
\newblock Constrained policy optimization.
\newblock In {\em International conference on machine learning}, pages 22--31.
  PMLR, 2017.

\bibitem{agarwal2022reincarnating}
Rishabh Agarwal, Max Schwarzer, Pablo~Samuel Castro, Aaron~C Courville, and
  Marc Bellemare.
\newblock Reincarnating reinforcement learning: Reusing prior computation to
  accelerate progress.
\newblock {\em Advances in Neural Information Processing Systems},
  35:28955--28971, 2022.

\bibitem{ajay2021opal}
Anurag Ajay, Aviral Kumar, Pulkit Agrawal, Sergey Levine, and Ofir Nachum.
\newblock {\{}OPAL{\}}: Offline primitive discovery for accelerating offline
  reinforcement learning.
\newblock In {\em International Conference on Learning Representations}, 2021.

\bibitem{akkaya2019solving}
Ilge Akkaya, Marcin Andrychowicz, Maciek Chociej, Mateusz Litwin, Bob McGrew,
  Arthur Petron, Alex Paino, Matthias Plappert, Glenn Powell, Raphael Ribas,
  et~al.
\newblock Solving rubik's cube with a robot hand.
\newblock {\em arXiv preprint arXiv:1910.07113}, 2019.

\bibitem{amit2020discount}
Ron Amit, Ron Meir, and Kamil Ciosek.
\newblock Discount factor as a regularizer in reinforcement learning.
\newblock In {\em International conference on machine learning}, pages
  269--278. PMLR, 2020.

\bibitem{andrychowicz2017hindsight}
Marcin Andrychowicz, Filip Wolski, Alex Ray, Jonas Schneider, Rachel Fong,
  Peter Welinder, Bob McGrew, Josh Tobin, OpenAI Pieter~Abbeel, and Wojciech
  Zaremba.
\newblock Hindsight experience replay.
\newblock {\em Advances in neural information processing systems}, 30, 2017.

\bibitem{auer2002using}
Peter Auer.
\newblock Using confidence bounds for exploitation-exploration trade-offs.
\newblock {\em Journal of Machine Learning Research}, 3(Nov):397--422, 2002.

\bibitem{bai2022monotonic}
Chenjia Bai, Ting Xiao, Zhoufan Zhu, Lingxiao Wang, Fan Zhou, Animesh Garg, Bin
  He, Peng Liu, and Zhaoran Wang.
\newblock Monotonic quantile network for worst-case offline reinforcement
  learning.
\newblock {\em IEEE Transactions on Neural Networks and Learning Systems},
  2022.

\bibitem{barekatain2021multipolar}
Mohammadamin Barekatain, Ryo Yonetani, and Masashi Hamaya.
\newblock Multipolar: multi-source policy aggregation for transfer
  reinforcement learning between diverse environmental dynamics.
\newblock In {\em Proceedings of the Twenty-Ninth International Conference on
  International Joint Conferences on Artificial Intelligence}, pages
  3108--3116, 2021.

\bibitem{barreto2018transfer}
Andre Barreto, Diana Borsa, John Quan, Tom Schaul, David Silver, Matteo Hessel,
  Daniel Mankowitz, Augustin Zidek, and Remi Munos.
\newblock Transfer in deep reinforcement learning using successor features and
  generalised policy improvement.
\newblock In {\em International Conference on Machine Learning}, pages
  501--510. PMLR, 2018.

\bibitem{bellemare2017distributional}
Marc~G Bellemare, Will Dabney, and R{\'e}mi Munos.
\newblock A distributional perspective on reinforcement learning.
\newblock In {\em International conference on machine learning}, pages
  449--458. PMLR, 2017.

\bibitem{bellman1957markovian}
Richard Bellman.
\newblock A markovian decision process.
\newblock {\em Journal of mathematics and mechanics}, pages 679--684, 1957.

\bibitem{berner2019dota}
Christopher Berner, Greg Brockman, Brooke Chan, Vicki Cheung, Przemys{\l}aw
  D{\k{e}}biak, Christy Dennison, David Farhi, Quirin Fischer, Shariq Hashme,
  Chris Hesse, et~al.
\newblock Dota 2 with large scale deep reinforcement learning.
\newblock {\em arXiv preprint arXiv:1912.06680}, 2019.

\bibitem{bharadhwaj2021conservative}
Homanga Bharadhwaj, Aviral Kumar, Nicholas Rhinehart, Sergey Levine, Florian
  Shkurti, and Animesh Garg.
\newblock Conservative safety critics for exploration.
\newblock In {\em International Conference on Learning Representations}, 2021.

\bibitem{borsa2018universal}
Diana Borsa, Andre Barreto, John Quan, Daniel~J. Mankowitz, Hado van Hasselt,
  Remi Munos, David Silver, and Tom Schaul.
\newblock Universal successor features approximators.
\newblock In {\em International Conference on Learning Representations}, 2019.

\bibitem{campos2021beyond}
V{\'\i}ctor Campos, Pablo Sprechmann, Steven~Stenberg Hansen, Andre Barreto,
  Steven Kapturowski, Alex Vitvitskyi, Adria~Puigdomenech Badia, and Charles
  Blundell.
\newblock Beyond fine-tuning: Transferring behavior in reinforcement learning.
\newblock In {\em ICML 2021 Workshop on Unsupervised Reinforcement Learning}.

\bibitem{chen2018improving}
Yi-Chun Chen, Mykel~J Kochenderfer, and Matthijs~TJ Spaan.
\newblock Improving offline value-function approximations for pomdps by
  reducing discount factors.
\newblock In {\em 2018 IEEE/RSJ International Conference on Intelligent Robots
  and Systems (IROS)}, pages 3531--3536. IEEE, 2018.

\bibitem{cheng2020policy}
Ching-An Cheng, Andrey Kolobov, and Alekh Agarwal.
\newblock Policy improvement via imitation of multiple oracles.
\newblock {\em Advances in Neural Information Processing Systems},
  33:5587--5598, 2020.

\bibitem{chow2015risk}
Yinlam Chow, Aviv Tamar, Shie Mannor, and Marco Pavone.
\newblock Risk-sensitive and robust decision-making: a cvar optimization
  approach.
\newblock {\em Advances in neural information processing systems}, 28, 2015.

\bibitem{chua2018deep}
Kurtland Chua, Roberto Calandra, Rowan McAllister, and Sergey Levine.
\newblock Deep reinforcement learning in a handful of trials using
  probabilistic dynamics models.
\newblock {\em Advances in neural information processing systems}, 31, 2018.

\bibitem{dabney2018implicit}
Will Dabney, Georg Ostrovski, David Silver, and R{\'e}mi Munos.
\newblock Implicit quantile networks for distributional reinforcement learning.
\newblock In {\em International conference on machine learning}, pages
  1096--1105. PMLR, 2018.

\bibitem{dabney2018distributional}
Will Dabney, Mark Rowland, Marc Bellemare, and R{\'e}mi Munos.
\newblock Distributional reinforcement learning with quantile regression.
\newblock In {\em Proceedings of the AAAI Conference on Artificial
  Intelligence}, volume~32, 2018.

\bibitem{dalal2021accelerating}
Murtaza Dalal, Deepak Pathak, and Russ~R Salakhutdinov.
\newblock Accelerating robotic reinforcement learning via parameterized action
  primitives.
\newblock {\em Advances in Neural Information Processing Systems},
  34:21847--21859, 2021.

\bibitem{eysenbach2018diversity}
Benjamin Eysenbach, Abhishek Gupta, Julian Ibarz, and Sergey Levine.
\newblock Diversity is all you need: Learning skills without a reward function.
\newblock In {\em International Conference on Learning Representations}, 2019.

\bibitem{fernandez2006probabilistic}
Fernando Fern{\'a}ndez and Manuela Veloso.
\newblock Probabilistic policy reuse in a reinforcement learning agent.
\newblock In {\em Proceedings of the fifth international joint conference on
  Autonomous agents and multiagent systems}, pages 720--727, 2006.

\bibitem{fu2020d4rl}
Justin Fu, Aviral Kumar, Ofir Nachum, George Tucker, and Sergey Levine.
\newblock D4rl: Datasets for deep data-driven reinforcement learning.
\newblock {\em arXiv preprint arXiv:2004.07219}, 2020.

\bibitem{fujimoto2018addressing}
Scott Fujimoto, Herke Hoof, and David Meger.
\newblock Addressing function approximation error in actor-critic methods.
\newblock In {\em International conference on machine learning}, pages
  1587--1596. PMLR, 2018.

\bibitem{gimelfarb2018reinforcement}
Michael Gimelfarb, Scott Sanner, and Chi-Guhn Lee.
\newblock Reinforcement learning with multiple experts: A bayesian model
  combination approach.
\newblock {\em Advances in neural information processing systems}, 31, 2018.

\bibitem{gimelfarb2021contextual}
Michael Gimelfarb, Scott Sanner, and Chi-Guhn Lee.
\newblock Contextual policy transfer in reinforcement learning domains via deep
  mixtures-of-experts.
\newblock In {\em Uncertainty in Artificial Intelligence}, pages 1787--1797.
  PMLR, 2021.

\bibitem{gregor2016variational}
Karol Gregor, Danilo~Jimenez Rezende, and Daan Wierstra.
\newblock Variational intrinsic control.
\newblock {\em arXiv preprint arXiv:1611.07507}, 2016.

\bibitem{haarnoja2018soft}
Tuomas Haarnoja, Aurick Zhou, Pieter Abbeel, and Sergey Levine.
\newblock Soft actor-critic: Off-policy maximum entropy deep reinforcement
  learning with a stochastic actor.
\newblock In {\em International conference on machine learning}, pages
  1861--1870. PMLR, 2018.

\bibitem{hafner2019learning}
Danijar Hafner, Timothy Lillicrap, Ian Fischer, Ruben Villegas, David Ha,
  Honglak Lee, and James Davidson.
\newblock Learning latent dynamics for planning from pixels.
\newblock In {\em International conference on machine learning}, pages
  2555--2565. PMLR, 2019.

\bibitem{howard1960dynamic}
Ronald~A Howard.
\newblock Dynamic programming and markov processes.
\newblock 1960.

\bibitem{hu2022role}
Hao Hu, Yiqin Yang, Qianchuan Zhao, and Chongjie Zhang.
\newblock On the role of discount factor in offline reinforcement learning.
\newblock In {\em International Conference on Machine Learning}, pages
  9072--9098. PMLR, 2022.

\bibitem{jiang2015dependence}
Nan Jiang, Alex Kulesza, Satinder Singh, and Richard Lewis.
\newblock The dependence of effective planning horizon on model accuracy.
\newblock In {\em Proceedings of the 2015 International Conference on
  Autonomous Agents and Multiagent Systems}, pages 1181--1189, 2015.

\bibitem{kalweit2019composite}
Gabriel Kalweit, Maria Huegle, and Joschka Boedecker.
\newblock Composite q-learning: Multi-scale q-function decomposition and
  separable optimization.
\newblock {\em arXiv preprint arXiv:1909.13518}, 2019.

\bibitem{kurenkov2020ac}
Andrey Kurenkov, Ajay Mandlekar, Roberto Martin-Martin, Silvio Savarese, and
  Animesh Garg.
\newblock Ac-teach: A bayesian actor-critic method for policy learning with an
  ensemble of suboptimal teachers.
\newblock In {\em Conference on Robot Learning}, pages 717--734. PMLR, 2020.

\bibitem{laroche2017multi}
Romain Laroche, Mehdi Fatemi, Joshua Romoff, and Harm van Seijen.
\newblock Multi-advisor reinforcement learning.
\newblock {\em arXiv preprint arXiv:1704.00756}, 2017.

\bibitem{lee2020learning}
Joonho Lee, Jemin Hwangbo, Lorenz Wellhausen, Vladlen Koltun, and Marco Hutter.
\newblock Learning quadrupedal locomotion over challenging terrain.
\newblock {\em Science robotics}, 5(47):eabc5986, 2020.

\bibitem{li2019context}
Siyuan Li, Fangda Gu, Guangxiang Zhu, and Chongjie Zhang.
\newblock Context-aware policy reuse.
\newblock In {\em Proceedings of the 18th International Conference on
  Autonomous Agents and MultiAgent Systems}, pages 989--997, 2019.

\bibitem{li2018optimal}
Siyuan Li and Chongjie Zhang.
\newblock An optimal online method of selecting source policies for
  reinforcement learning.
\newblock In {\em Proceedings of the AAAI Conference on Artificial
  Intelligence}, volume~32, 2018.

\bibitem{lin2020rd}
Zichuan Lin, Derek Yang, Li~Zhao, Tao Qin, Guangwen Yang, and Tie-Yan Liu.
\newblock Rd $^{2}$: Reward decomposition with representation decomposition.
\newblock {\em Advances in Neural Information Processing Systems},
  33:11298--11308, 2020.

\bibitem{merel2018neural}
Josh Merel, Leonard Hasenclever, Alexandre Galashov, Arun Ahuja, Vu~Pham, Greg
  Wayne, Yee~Whye Teh, and Nicolas Heess.
\newblock Neural probabilistic motor primitives for humanoid control.
\newblock In {\em International Conference on Learning Representations}, 2019.

\bibitem{munos2008finite}
R{\'e}mi Munos and Csaba Szepesv{\'a}ri.
\newblock Finite-time bounds for fitted value iteration.
\newblock {\em Journal of Machine Learning Research}, 9(5), 2008.

\bibitem{nasiriany2019planning}
Soroush Nasiriany, Vitchyr Pong, Steven Lin, and Sergey Levine.
\newblock Planning with goal-conditioned policies.
\newblock {\em Advances in Neural Information Processing Systems}, 32, 2019.

\bibitem{peng2019mcp}
Xue~Bin Peng, Michael Chang, Grace Zhang, Pieter Abbeel, and Sergey Levine.
\newblock Mcp: Learning composable hierarchical control with multiplicative
  compositional policies.
\newblock {\em Advances in Neural Information Processing Systems}, 32, 2019.

\bibitem{petrik2008biasing}
Marek Petrik and Bruno Scherrer.
\newblock Biasing approximate dynamic programming with a lower discount factor.
\newblock {\em Advances in neural information processing systems}, 21, 2008.

\bibitem{Qureshi2020Composing}
Ahmed~H. Qureshi, Jacob~J. Johnson, Yuzhe Qin, Taylor Henderson, Byron Boots,
  and Michael~C. Yip.
\newblock Composing task-agnostic policies with deep reinforcement learning.
\newblock In {\em International Conference on Learning Representations}, 2020.

\bibitem{rakelly2019efficient}
Kate Rakelly, Aurick Zhou, Chelsea Finn, Sergey Levine, and Deirdre Quillen.
\newblock Efficient off-policy meta-reinforcement learning via probabilistic
  context variables.
\newblock In {\em International conference on machine learning}, pages
  5331--5340. PMLR, 2019.

\bibitem{reinke2017average}
Chris Reinke, Eiji Uchibe, and Kenji Doya.
\newblock Average reward optimization with multiple discounting reinforcement
  learners.
\newblock In {\em Neural Information Processing: 24th International Conference,
  ICONIP 2017, Guangzhou, China, November 14-18, 2017, Proceedings, Part I 24},
  pages 789--800. Springer, 2017.

\bibitem{romoff2019separating}
Joshua Romoff, Peter Henderson, Ahmed Touati, Emma Brunskill, Joelle Pineau,
  and Yann Ollivier.
\newblock Separating value functions across time-scales.
\newblock In {\em International Conference on Machine Learning}, pages
  5468--5477. PMLR, 2019.

\bibitem{rosman2016bayesian}
Benjamin Rosman, Majd Hawasly, and Subramanian Ramamoorthy.
\newblock Bayesian policy reuse.
\newblock {\em Machine Learning}, 104:99--127, 2016.

\bibitem{rummery1994line}
Gavin~A Rummery and Mahesan Niranjan.
\newblock {\em On-line Q-learning using connectionist systems}, volume~37.
\newblock University of Cambridge, Department of Engineering Cambridge, UK,
  1994.

\bibitem{schaul2015universal}
Tom Schaul, Daniel Horgan, Karol Gregor, and David Silver.
\newblock Universal value function approximators.
\newblock In {\em International conference on machine learning}, pages
  1312--1320. PMLR, 2015.

\bibitem{schulman2017proximal}
John Schulman, Filip Wolski, Prafulla Dhariwal, Alec Radford, and Oleg Klimov.
\newblock Proximal policy optimization algorithms.
\newblock {\em arXiv preprint arXiv:1707.06347}, 2017.

\bibitem{sekar2020planning}
Ramanan Sekar, Oleh Rybkin, Kostas Daniilidis, Pieter Abbeel, Danijar Hafner,
  and Deepak Pathak.
\newblock Planning to explore via self-supervised world models.
\newblock In {\em International Conference on Machine Learning}, pages
  8583--8592. PMLR, 2020.

\bibitem{Sharma2020Dynamics-Aware}
Archit Sharma, Shixiang Gu, Sergey Levine, Vikash Kumar, and Karol Hausman.
\newblock Dynamics-aware unsupervised discovery of skills.
\newblock In {\em International Conference on Learning Representations}, 2020.

\bibitem{sherstan2020gamma}
Craig Sherstan, Shibhansh Dohare, James MacGlashan, Johannes G{\"u}nther, and
  Patrick~M Pilarski.
\newblock Gamma-nets: Generalizing value estimation over timescale.
\newblock In {\em Proceedings of the AAAI Conference on Artificial
  Intelligence}, volume~34, pages 5717--5725, 2020.

\bibitem{strouse2022learning}
DJ~Strouse, Kate Baumli, David Warde-Farley, Volodymyr Mnih, and
  Steven~Stenberg Hansen.
\newblock Learning more skills through optimistic exploration.
\newblock In {\em International Conference on Learning Representations}, 2022.

\bibitem{sutton2018reinforcement}
Richard~S Sutton and Andrew~G Barto.
\newblock {\em Reinforcement learning: An introduction}.
\newblock MIT press, 2018.

\bibitem{tang2021taylor}
Yunhao Tang, Mark Rowland, R{\'e}mi Munos, and Michal Valko.
\newblock Taylor expansion of discount factors.
\newblock In {\em International Conference on Machine Learning}, pages
  10130--10140. PMLR, 2021.

\bibitem{uchendu2022jump}
Ikechukwu Uchendu, Ted Xiao, Yao Lu, Banghua Zhu, Mengyuan Yan, Jos{\'e}phine
  Simon, Matthew Bennice, Chuyuan Fu, Cong Ma, Jiantao Jiao, et~al.
\newblock Jump-start reinforcement learning.
\newblock {\em arXiv preprint arXiv:2204.02372}, 2022.

\bibitem{van2016deep}
Hado Van~Hasselt, Arthur Guez, and David Silver.
\newblock Deep reinforcement learning with double q-learning.
\newblock In {\em Proceedings of the AAAI conference on artificial
  intelligence}, volume~30, 2016.

\bibitem{van2019use}
Hado~P Van~Hasselt, Matteo Hessel, and John Aslanides.
\newblock When to use parametric models in reinforcement learning?
\newblock {\em Advances in Neural Information Processing Systems}, 32, 2019.

\bibitem{van2017hybrid}
Harm Van~Seijen, Mehdi Fatemi, Joshua Romoff, Romain Laroche, Tavian Barnes,
  and Jeffrey Tsang.
\newblock Hybrid reward architecture for reinforcement learning.
\newblock {\em Advances in Neural Information Processing Systems}, 30, 2017.

\bibitem{vezzani2022skills}
Giulia Vezzani, Dhruva Tirumala, Markus Wulfmeier, Dushyant Rao, Abbas
  Abdolmaleki, Ben Moran, Tuomas Haarnoja, Jan Humplik, Roland Hafner, Michael
  Neunert, et~al.
\newblock Skills: Adaptive skill sequencing for efficient temporally-extended
  exploration.
\newblock {\em arXiv preprint arXiv:2211.13743}, 2022.

\bibitem{watkins1992q}
Christopher~JCH Watkins and Peter Dayan.
\newblock Q-learning.
\newblock {\em Machine learning}, 8:279--292, 1992.

\bibitem{wolczyk2022disentangling}
Maciej Wolczyk, Micha{\l} Zaj{\k{a}}c, Razvan Pascanu, {\L}ukasz Kuci{\'n}ski,
  and Piotr Mi{\l}o{\'s}.
\newblock Disentangling transfer in continual reinforcement learning.
\newblock {\em Advances in Neural Information Processing Systems},
  35:6304--6317, 2022.

\bibitem{xu2018meta}
Zhongwen Xu, Hado~P van Hasselt, and David Silver.
\newblock Meta-gradient reinforcement learning.
\newblock {\em Advances in neural information processing systems}, 31, 2018.

\bibitem{yang2020multi}
Ruihan Yang, Huazhe Xu, Yi~Wu, and Xiaolong Wang.
\newblock Multi-task reinforcement learning with soft modularization.
\newblock {\em Advances in Neural Information Processing Systems},
  33:4767--4777, 2020.

\bibitem{yang2023behavior}
Rushuai Yang, Chenjia Bai, Hongyi Guo, Siyuan Li, Bin Zhao, Zhen Wang, Peng
  Liu, and Xuelong Li.
\newblock Behavior contrastive learning for unsupervised skill discovery, 2023.

\bibitem{yang2020efficient}
Tianpei Yang, Jianye Hao, Zhaopeng Meng, Zongzhang Zhang, Yujing Hu, Yingfeng
  Chen, Changjie Fan, Weixun Wang, Zhaodong Wang, and Jiajie Peng.
\newblock Efficient deep reinforcement learning through policy transfer.
\newblock In {\em AAMAS}, pages 2053--2055, 2020.

\bibitem{yu2020meta}
Tianhe Yu, Deirdre Quillen, Zhanpeng He, Ryan Julian, Karol Hausman, Chelsea
  Finn, and Sergey Levine.
\newblock Meta-world: A benchmark and evaluation for multi-task and meta
  reinforcement learning.
\newblock In {\em Conference on robot learning}, pages 1094--1100. PMLR, 2020.

\bibitem{zhang2023efficient}
Grace Zhang, Ayush Jain, Injune Hwang, Shao-Hua Sun, and Joseph~J Lim.
\newblock Efficient multi-task reinforcement learning via selective behavior
  sharing.
\newblock {\em arXiv preprint arXiv:2302.00671}, 2023.

\bibitem{zhangpolicy}
Haichao Zhang, Wei Xu, and Haonan Yu.
\newblock Policy expansion for bridging offline-to-online reinforcement
  learning.
\newblock In {\em The Eleventh International Conference on Learning
  Representations}.

\bibitem{zhangcup}
Jin Zhang, Siyuan Li, and Chongjie Zhang.
\newblock Cup: Critic-guided policy reuse.
\newblock In {\em Advances in Neural Information Processing Systems}, 2022.

\bibitem{zhang2021distributional}
Pushi Zhang, Xiaoyu Chen, Li~Zhao, Wei Xiong, Tao Qin, and Tie-Yan Liu.
\newblock Distributional reinforcement learning for multi-dimensional reward
  functions.
\newblock {\em Advances in Neural Information Processing Systems},
  34:1519--1529, 2021.

\end{thebibliography}

\clearpage

\appendix

\section{Algorithm Description}
\label{app:pseudocode}
The pseudocode of our algorithm is presented in Algorithm~\ref{alg:pseudocode}. During the online interactions, we perform value-guided behavior planning by periodically switching the behavior policy according to Eq.~(\ref{eq:q_ucb_guided_behavior_policy}). The $\eta$ in Algorithm~\ref{alg:pseudocode} denotes the selected behavior policy and is chosen from the set of policies $\{\pi\}\cup\{\mu_i\}_{i=1}^K$, rather than an individual policy module. To evaluate the behaviors of all policies, we perform temporal difference learning to optimize the value functions. The task policy $\pi$ is optimized via Soft-Actor Critic algorithm~\cite{haarnoja2018soft} based on the value function $Q_\theta^{\pi,\gamma}$.

\begin{algorithm}[tbh]
    \setstretch{1}
    \caption{Selective Myopic bEhavior Control~(SMEC)}
    \label{alg:pseudocode}
    {\bf Input:}  Prior policies $\{\mu_i\}_{i=1}^K$, current task $\mathcal{M}$.\\
    {\bf Initialization:} Task policy $\pi$, value function $Q_{\theta}$, total policy selection counts $T:=0$, policy utilization counts $N^1 := \textit{\textbf{0}} \in \mathbb{R}^{K+1}$, policy transformation count matrix $N^2 := \textit{\textbf{0}} \in \mathbb{R}^{(K+1)\times (K+1)}$, ucb coefficient $c$, short-term horizon length $h$, short-term discount factor $\bar{\gamma}:=\epsilon^{\frac{1}{h}}$, batch size $B$, replay buffer $D$.
    \begin{algorithmic}[1]
        \setstretch{1.2}
        \STATE \# \texttt{Initialize behavior policy}
        \STATE $\eta \leftarrow \pi$ $\qquad$
        \FOR{t = 0, 1, 2, \dots}
            \STATE Sample transition $(s,a,s',r)$ from $\mathcal{M}$ using $\eta$
            \STATE $D~\leftarrow~D~\cup~(s,a,s',r)$
            \IF{ $(t+1)~\%~h == 0$}
                \STATE \# \texttt{Policy switch}
                \STATE $\eta^* \quad \leftarrow \underset{\nu\in\{\pi\}\cup\{\mu_i\}_{i=1}^K}{\arg\max} \left[ Q^\nu_{\theta}(s ,a) + c\cdot \sqrt{\dfrac{\log(2T)}{N^1_\nu + N^2_{\eta\rightarrow\nu}}}\right]$
                \STATE \# \texttt{Update ucb parameters}
                \STATE $T\qquad\quad\leftarrow T + 1$
                \STATE $N^1_{\eta^*}\qquad\leftarrow N^1_{\eta^*} + 1$
                \STATE $N^2_{\eta \rightarrow \eta^*}~~~\leftarrow N^2_{\eta \rightarrow \eta^*} + 1$
                \STATE \# \texttt{Update behavior policy}
                \STATE $\eta ~~~\leftarrow \eta^*$
            \ENDIF
            \STATE Sample a batch of transitions $\{(s,a,r,s')\}^B$ from replay buffer $D$
            \STATE Sample the batch of actions $\{a_\nu'\}^B$ form $\nu(\cdot|s')$ for each policy $\nu \in \{\pi\}\cup\{\mu_i\}_{i=1}^K$
            \STATE Train value function $Q_\theta$ by minimizing:
            \begin{align*}
                J(\theta):=&\dfrac{1}{2B}\sum_{\{(s,a,r,s')\}^B}\Bigl[\left(Q_{\theta}^{\pi,\gamma}(s,a) - \mathcal{T}^{\pi}_{\gamma}(s',r)\right)^2 + \sum_{i=1}^K \left(Q_{\theta}^{\mu_i,\bar{\gamma}}(s,a) - \mathcal{T}^{\mu_i}_{\bar{\gamma}}(s',r)\right)^2\Bigr]\\
                \text{where}~~&\mathcal{T}^{\pi}_{\gamma}(s',r) := r + \gamma~ Q_{\bar{\theta}}^{\pi,\gamma}(s',a_\pi'),\nonumber \qquad\qquad~~~\text{(long-term task policy operator)}\\
                \qquad\quad~~&\mathcal{T}^{\mu_i}_{\bar{\gamma}}(s',r) := r + \bar{\gamma}~ Q_{\bar{\theta}}^{\mu_i,\bar{\gamma}}(s',a_{\mu_i}'),\nonumber \qquad\quad~\text{(short-term prior policy operator)}
            \end{align*}
            \STATE Train task policy $\pi$ with the value function $Q_\theta^{\pi,\gamma}$ via SAC
        \ENDFOR
    \end{algorithmic}
\end{algorithm}

\section{Theoretical Results}
\label{app:theory}

\subsection{Performance Guarantee of the Behavior Policy}
\label{app:theory:performance}
This section presents the proofs of our main results, Theorem~\ref{theorem:value_bound} and Theorem~\ref{thm:perf_bound}. Specifically, we provide a rigorous analysis of the induced behavior policy $\eta$ using value-guided behavior planning (\textit{i.e.},~Eq(\ref{eq:q_guided_behavior_policy})). In Theorem~\ref{theorem_app:value_bound}, we prove that the performance of the behavior policy after a single policy switch is guaranteed to outperform the performance of the task policy. In Theorem~\ref{thm_app:perf_bound}, we prove that the behavior policy induced by the single switched sub-trajectory achieves a lower-bound performance compared with the task policy.

\begin{theorem}
    \label{theorem_app:value_bound}
    Following the behavior policy induced by Eq.~(\ref{eq:q_guided_behavior_policy}), when the prior policy $\bar{\mu}:= \arg\max_{\mu\in\{\mu_i\}_{i=1}^K}V^\mu_{\bar{\gamma}}(s_j)$ meets $V^{\bar{\mu}}_{\bar{\gamma}}(s_j)\geq V^\pi_{\gamma}(s_j), \forall s_j \in S, \exists~j \in [0, h, 2h, \dots]$, $\bar{\gamma}<\gamma$, and the policy $\eta$ is fixed after the switch. the induced value of $\eta$ can be bounded as follows:
    \begin{equation*}
        V_\gamma^{\eta}(s_j) - V_\gamma^{\pi}(s_j) ~\geq~ \dfrac{\gamma-\bar{\gamma}}{(1-\gamma)(1-\bar{\gamma})}R_{\text{max}}>0.
    \end{equation*}
\end{theorem}

\begin{proof}
    At the switch point $t=j$ with state $s_j, j \in [0, h, 2h, \dots]$, since the short-term behavior value of the most performant prior policy $\bar{\mu}:= \arg\max_{\mu\in\{\mu_i\}_{i=1}^K}V^\mu_{\bar{\gamma}}(s_j)$ surpasses the long-term behavior value of the task policy, we select the prior policy as the behavior policy for the subsequent interactions following the value guidance. The condition can be expressed as $V^{\bar{\mu}}_{\bar{\gamma}}(s_j)\geq V^\pi_{\gamma}(s_j), \forall s_j \in S, \exists~j \in [0, h, 2h, \dots]$. Therefore, we set the behavior policy after the switch point as 
    \begin{align*}
      &\eta(\cdot|s_j) = \bar{\mu}(\cdot|s_j):=\arg\max_{\{\mu_i\}_{i=1}^K}V^\mu_{\bar{\gamma}}(s_j)~(\cdot|s_t).
    \end{align*}
    
    Thus we can transform the value difference between the behavior policy and the task policy as follows:
    \begin{equation*}
        V^\eta_\gamma(s_j) - V^\pi_\gamma(s_j) = V^{\bar{\mu}}_\gamma(s_j) - V^\pi_\gamma(s_j).     
    \end{equation*}
    Since $V^{\bar{\mu}}_{\bar{\gamma}}(s_j)\geq V^\eta_{\gamma}(s_j)$ always holds under the condition, by assuming the behavior policy is fixed after the policy switch, we can bound the performance difference between two policies starting from the state $s_j$ as follows:
    \begin{align}
        V^{\eta}_\gamma(s_j) - V^\pi_\gamma(s_j) =&~ V^{\bar{\mu}}_\gamma(s_j) - V^{\bar{\mu}}_{\bar{\gamma}}(s_j) + V^{\bar{\mu}}_{\bar{\gamma}}(s_j) - V^\pi_\gamma(s_j)\nonumber\\
        \geq&~ \dfrac{\gamma-\bar{\gamma}}{(1-\gamma)(1-\bar{\gamma})}R_{\text{max}} + V^{\bar{\mu}}_{\bar{\gamma}}(s_j) - V^\pi_\gamma(s_j) & (\text{Lemma}~\ref{lemma:discount_diff}) \nonumber\\
        \geq & \dfrac{\gamma-\bar{\gamma}}{(1-\gamma)(1-\bar{\gamma})}R_{\text{max}} + 0 & \left(V^{\bar{\mu}}_{\bar{\gamma}}(s_j)\geq V^\pi_{\gamma}(s_j) \right)\nonumber\\
        =&~ \dfrac{\gamma-\bar{\gamma}}{(1-\gamma)(1-\bar{\gamma})}R_{\text{max}}. \nonumber\\
        >&~0 & (\gamma>\bar{\gamma})
    \end{align}

    The result demonstrates that the selected behavior policy following the value guidance is guaranteed to outperform the task policy after the switch point. 
\end{proof}

\begin{theorem}
    \label{thm_app:perf_bound}
    When there is only one sub-trajectory from $kt$ to $(k+1)h$ during which a prior policy $\bar{\mu}$ is selected, which means no prior policy $\mu \in \{\mu_i\}_{i=1}^K$ satisfies $V_{\bar{\gamma}}^\mu(s_t) \geq V_\gamma^\pi(s_t)$ except $t \in [kh, (k+1)h)$. The performance difference between the behavior policy $\eta$ induced by Eq.~(\ref{eq:q_guided_behavior_policy}) and the task policy $\pi$ is bounded as follows:
    $$
    J_\gamma(\eta) - J_\gamma(\pi) \geq \gamma^{kh} \dfrac{\gamma-\bar{\gamma}}{(1-\gamma)(1-\bar{\gamma})}R_{\text{max}} - \gamma^{(k+1)h} \dfrac{R_{\text{max}}}{(1-\gamma)^2}\|\bar{\mu}-\pi\|_\infty,
    $$
    where $\|\bar{\mu}-\pi\|_\infty := \underset{s\in S}{\sup}\underset{A}{\sum}|\bar{\mu}(a|s) - \pi(a|s)|$, and $\bar{\mu} = \arg\max_{\mu\in\{\mu_i\}_{i=1}^K} V^\mu_{\bar{\gamma}} (s_{kh}),~\forall s_{kh}\in \{s\in S | \mathbb{P}_{kh}^\pi(s) > 0\}$.
\end{theorem}

\begin{figure}[h]
    \centering
    \includegraphics[width=\textwidth]{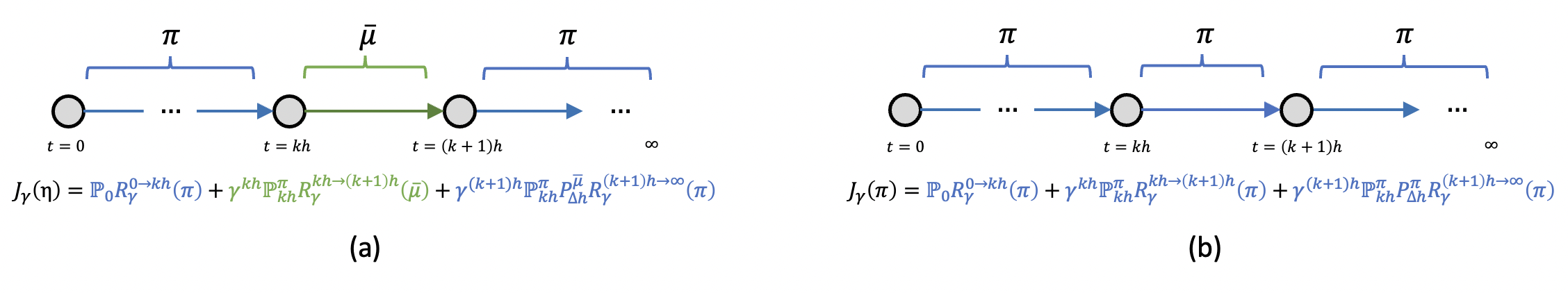}
    \vspace{-1.2em}
    \caption{(a)~The episode illustration and the performance of the behavior policy $\eta$ under the condition in Theorem~\ref{thm_app:perf_bound} with a single switched sub-trajectory. (b)~The episode illustration and the performance of the task policy $\pi$.}
    \label{fig:thm_illustration}
    \vspace{-0.5em}
\end{figure}

\begin{proof}
    Since the behavior policy induced by the value-guided selection is non-stationary within the episode~(\textit{i.e.}, $\pi$ in $t\in[0,~kh]$,  $\bar{\mu}$ in $t\in[kh,~(k+1)h]$, and $\pi$ in $t\in[(k+1)h,\infty]$), we first decompose the policy performance $J_\gamma(\nu), \forall \nu \in \{\pi\}\cup\{\mu_i\}_{i=1}^K$ into three components:
    \begin{align*}
        J_\gamma(\nu) &= \mathbb{P}_0 R_\gamma^{0\rightarrow kh}(\nu) + \gamma^{kh} \mathbb{P}_{kh}^\nu R_\gamma^{kh \rightarrow (k+1)h}(\nu) + \gamma^{(k+1)h} \mathbb{P}_{kh}^\nu R_\gamma^{(k+1)h \rightarrow \infty}(\nu),
    \end{align*}
    where $\mathbb{P}_0$ denotes the initial state distribution, $\mathbb{P}_t^\nu$ denotes the state distribution at time $t$ induced by the policy $\nu$ starting from the initial state distribution, and each component can be defined as
    \begin{align*}
        \mathbb{P}^\nu_{i}R_\gamma^{i\rightarrow j}(\nu) := \underset{s_i\in S}{\sum} \mathbb{P}_i^\nu(s_i) \left( V_\gamma^\nu(s_i) - \gamma^{j-i} \underset{s_j\in S}{\sum} P_{\Delta (j-i)}^\nu (s_j|s_i) V_\gamma^\nu(s_j) \right).
    \end{align*}

    Since during sub-trajectory from $kh$ to $(k+1)h$ the prior policy $\bar{\mu}$ is selected, as shown in Figure~\ref{fig:thm_illustration}~(a), the performance of the behavior policy $\eta$ can be defined as:
    \begin{align*}
        J_\gamma(\eta) &= \mathbb{P}_0 R_\gamma^{0\rightarrow kh}(\pi) + \gamma^{kh} \mathbb{P}_{kh}^\pi R_\gamma^{kh \rightarrow (k+1)h}(\bar{\mu}) + \gamma^{(k+1)h} 
 \mathbb{P}_{kh}^\pi P_{\Delta h}^{\bar{\mu}} R_\gamma^{(k+1)h \rightarrow \infty}(\pi)\\
        &=  \mathbb{P}_0 R_\gamma^{0\rightarrow kh}(\pi) \\
        &\quad~ + \gamma^{kh}  \underset{s_{kh}\in S}{\sum} \mathbb{P}_{kh}^\pi(s_{kh}) \left(V_\gamma^{\bar{\mu}} (s_{kh}) - \gamma^{h} \underset{s_{(k+1)h}\in S}{\sum} P_{\Delta h}^{\bar{\mu}}(s_{(k+1)h}|s_{kh}) V_\gamma^{\bar{\mu}} (s_{(k+1)h}) \right)\\
        &\quad~ + \gamma^{(k+1)h}\underset{s_{kh}\in S}{\sum} \mathbb{P}_{kh}^\pi(s_{kh}) \left(\underset{s_{(k+1)h}\in S}{\sum} P_{\Delta h}^{\bar{\mu}}(s_{(k+1)h}|s_{kh}) V_\gamma^{\pi}(s_{(k+1)h}) \right)\\
        &= \underbrace{\mathbb{P}_0 R_\gamma^{0\rightarrow kh}(\pi)}_{(a1)} + \underbrace{\gamma^{kh}\underset{s_{kh}\in S}{\sum} \mathbb{P}_{kh}^\pi(s_{kh}) V_\gamma^{\bar{\mu}} (s_{kh})}_{(a2)}\\
        &\quad~ + \underbrace{\gamma^{(k+1)h}\underset{s_{kh}\in S}{\sum} \mathbb{P}_{kh}^\pi(s_{kh})\underset{s_{(k+1)h}\in S}{\sum} P_{\Delta h}^{\bar{\mu}}(s_{(k+1)h}|s_{kh})\left(V_\gamma^{\pi}(s_{(k+1)h}) - V_\gamma^{\bar{\mu}}(s_{(k+1)h})\right)}_{(a3)},\\
    \end{align*}
    where $\bar{\mu} = \arg\max_{\{\mu_i\}_{i=1}^K} V^\mu_{\bar{\gamma}} (s_{kh}),~\forall s_{kh}\in \{s\in S | \mathbb{P}_{kh}^\pi(s) > 0\}$, $P_{\Delta h}^\pi(s_{(k+1)h}|s_{kh}):= \sum_{(a_{kh},\dots,s_{(k+1)h})} \left(\prod_{i=0}^{h-1}\pi(a_{kh+i}|s_{kh+i})\mathcal{P}(s_{kh+i+1}|s_{kh+i},a_{kh+i})\right)$ denotes the transition probability from state $s_{kh}$ to state $s_{(k+1)h}$ by rollouting policy $\pi$ for $h$ steps.

    As shown in Figure~\ref{fig:thm_illustration}~(b), the induced performance by simply using the task policy can also be decomposed as follows:
    \begin{align*}
        J_\gamma(\pi) &= \mathbb{P}_0 R_\gamma^{0\rightarrow kh}(\pi) + \gamma^{kh}\mathbb{P}_{kh}^\pi R_\gamma^{kh \rightarrow (k+1)h}(\pi) + \gamma^{(k+1)h} \mathbb{P}_{kh}^\pi P_{\Delta h}^{\pi} R_\gamma^{(k+1)h \rightarrow \infty}(\pi)\\
        &=  \mathbb{P}_0 R_\gamma^{0\rightarrow kh}(\pi) + \gamma^{kh}  \underset{s_{kh}\in S}{\sum} \mathbb{P}_{kh}^\pi(s_{kh}) \left(V_\gamma^{\pi} (s_{kh}) - \gamma^{h} \underset{s_{(k+1)h}\in S}{\sum} P_{\Delta h}^{\pi}(s_{(k+1)h}|s_{kh}) V_\gamma^{\pi} (s_{(k+1)h}) \right)\\
        &\qquad\qquad\quad\quad~~ + \gamma^{(k+1)h}\underset{s_{kh}\in S}{\sum} \mathbb{P}_{kh}^\pi(s_{kh}) \left( \underset{s_{(k+1)h}\in S}{\sum} P_{\Delta h}^{\pi}(s_{(k+1)h}|s_{kh}) V_\gamma^{\pi}(s_{(k+1)h}) \right)\\
        &= \underbrace{\mathbb{P}_0 R_\gamma^{0\rightarrow kh}(\pi)}_{(b1)} + \underbrace{\gamma^{kh}  \underset{s_{kh}\in S}{\sum} \mathbb{P}_{kh}^\pi(s_{kh}) V_\gamma^{\pi} (s_{kh})}_{(b2)}.
    \end{align*}

    Thus, the performance difference between the behavior policy $\eta$ and the task policy $\pi$ can be bound as follows:
    \begin{align*}
        &J_\gamma(\eta) - J_\gamma(\pi) = \bigl[(a1) - (a2)\bigr] + \bigl[(a2) - (b2)\bigr] + (a3)\\
        &~= 0 +\gamma^{kh} \underset{s_{kh}\in S}{\sum} \mathbb{P}_{kh}^\pi(s_{kh}) \left( V_\gamma^{\bar{\mu}} (s_{kh}) - V_\gamma^{\pi} (s_{kh})\right)\\
        &\quad + \gamma^{(k+1)h} \underset{s_{kh}\in S}{\sum} \mathbb{P}_{kh}^\pi(s_{kh})\underset{s_{(k+1)h}\in S}{\sum} P_{\Delta h}^{\bar{\mu}}(s_{(k+1)h}|s_{kh}) \left(V_\gamma^{\pi} (s_{(k+1)h}) - V_\gamma^{\bar{\mu}} (s_{(k+1)h})\right)\\
        &~\overset{(i)}{\geq} \gamma^{kh} \dfrac{(\gamma-\bar{\gamma})R_{\text{max}}}{(1-\gamma)(1-\bar{\gamma})} + \gamma^{(k+1)h} \underset{s_{kh}\in S}{\sum} \mathbb{P}_{kh}^\pi(s_{kh})\underset{s_{(k+1)h}\in S}{\sum} P_{\Delta h}^{\bar{\mu}}(s_{(k+1)h}|s_{kh}) \left(V_\gamma^{\pi} (s_{(k+1)h}) - V_\gamma^{\bar{\mu}} (s_{(k+1)h})\right)\\
        &~\overset{(ii)}{\geq} \gamma^{kh} \dfrac{\gamma-\bar{\gamma}}{(1-\gamma)(1-\bar{\gamma})}R_{\text{max}}\\
        &\quad - \gamma^{(k+1)h}\Bigl(\underset{s_{kh}\in S}{\sum} \mathbb{P}_{kh}^\pi(s_{kh})\underset{s_{(k+1)h}\in S}{\sum}P_{\Delta h}^{\bar{\mu}}(s_{(k+1)h}|s_{kh}) \dfrac{2R_{\text{max}}}{(1-\gamma)^2}~\mathbb{E}_{d^{\bar{\mu}}(\cdot|s_0=s_{(k+1)h})}\left[D_{TV}\left[\bar{\mu}(\cdot|s)\|\pi(\cdot|s)\right]\right]\Bigr)\\
        &~\overset{(iii)}{\geq} \gamma^{kh} \dfrac{\gamma-\bar{\gamma}}{(1-\gamma)(1-\bar{\gamma})}R_{\text{max}} - \gamma^{(k+1)h} \dfrac{R_{\text{max}}}{(1-\gamma)^2} \|\bar{\mu}-\pi\|_\infty, 
    \end{align*}
    where steps $(i)$ holds by Theorem~\ref{theorem_app:value_bound}, $(ii)$ holds by Lemma~\ref{lemma:value_diff}, and $(iii)$ holds by
    \begin{align*}
        &\qquad\underset{s_{kh}\in S}{\sum} \mathbb{P}_{kh}^\pi(s_{kh})\underset{s_{(k+1)h}\in S}{\sum}P_{\Delta h}^{\bar{\mu}}(s_{(k+1)h}|s_{kh}) \dfrac{2R_{\text{max}}}{(1-\gamma)^2}~\mathbb{E}_{d^{\bar{\mu}}(\cdot|s_0=s_{(k+1)h})}\left[D_{TV}\left[\bar{\mu}(\cdot|s)\|\pi(\cdot|s)\right]\right]\\
        &= \underset{s_{kh}\in S}{\sum} \mathbb{P}_{kh}^\pi(s_{kh})\underset{s_{(k+1)h}\in S}{\sum}P_{\Delta h}^{\bar{\mu}}(s_{(k+1)h}|s_{kh}) \dfrac{2R_{\text{max}}}{(1-\gamma)^2}~\mathbb{E}_{d^{\bar{\mu}}(\cdot|s_0=s_{(k+1)h})}\left[\dfrac{1}{2}\underset{A}{\sum}|\bar{\mu}(a|s) - \pi(a|s)|\right]\\
        &= \underset{s_{kh}\in S}{\sum} \mathbb{P}_{kh}^\pi(s_{kh})\underset{s_{(k+1)h}\in S}{\sum}P_{\Delta h}^{\bar{\mu}}(s_{(k+1)h}|s_{kh}) \dfrac{R_{\text{max}}}{(1-\gamma)^2}~\underset{S}{\sum} d^{\bar{\mu}}(s|s_0=s_{(k+1)h}) \|\bar{\mu}(\cdot|s) - \pi(\cdot|s)\|_1\\
        &\leq \underset{s_{kh}\in S}{\sum} \mathbb{P}_{kh}^\pi(s_{kh})\underset{s_{(k+1)h}\in S}{\sum}P_{\Delta h}^{\bar{\mu}}(s_{(k+1)h}|s_{kh}) \dfrac{R_{\text{max}}}{(1-\gamma)^2}~\underset{S}{\sum} d^{\bar{\mu}}(s|s_0=s_{(k+1)h}) \underset{S}{\sup} \|\bar{\mu}(\cdot|s) - \pi(\cdot|s)\|_1\\
        &= \dfrac{R_{\text{max}}}{(1-\gamma)^2} \|\bar{\mu} - \pi\|_\infty.
    \end{align*}
\end{proof}

\subsection{Useful Lemmas}
\label{app:theory:useful_lemmas}
This section provides proof of several lemmas used for our theoretical results. The first two lemmas are adopted from Lemma~1 in \cite{jiang2015dependence} and Lemma~3 in \cite{achiam2017constrained}, respectively, and the proof is essentially the same as the original paper. The last lemma provides value difference bound over two different policies starting from the same state.

\begin{lemma}
    {\normalfont \textbf{(Lemma~1 in \cite{jiang2015dependence})}}
    \label{lemma:discount_diff}
    For any MDP $M$ with rewards in $[0, R_{\text{max}}]$, $\forall \pi: \mathcal{S} \rightarrow \mathcal{A}$ and $\gamma_1 \neq \gamma_2$,
    \begin{equation*}
        V^\pi_{\gamma_1}(s) - V^\pi_{\gamma_2}(s) \leq \|V^\pi_{\gamma_1} - V^\pi_{\gamma_2}\|_{\infty} \leq \dfrac{\gamma_1 - \gamma_2}{(1-\gamma_1)(1-\gamma_2)}R_\text{max}
    \end{equation*}
\end{lemma}

\begin{proof}
    Letting $[P^\pi]$ denotes the transition probability matrix for policy $\pi$ (matrix form of $P^\pi(\cdot|\cdot),~\forall s,s'\in S$), we have
    \begin{align*}
        V^\pi_{\gamma_1}(s) - V^\pi_{\gamma_2}(s)~\leq~\|V^\pi_{\gamma_1} - V^\pi_{\gamma_2}\|_{\infty}~=&~\left\|\sum_{t=0}^{\infty} (\gamma_1^t - \gamma_2^t) [P^\pi]^t R^\pi\right\|_{\infty}\\
        \leq& ~\sum_{t=0}^\infty (\gamma_1^t - \gamma_2^t) R_{\text{max}} \\
        =& ~\dfrac{\gamma_1 - \gamma_2}{(1-\gamma_1)(1-\gamma_2)} R_{\text{max}}
    \end{align*}
\end{proof}

\begin{lemma}
    \label{lemma:state_dis_diff}
    {\normalfont \textbf{(Lemma~3 in \cite{achiam2017constrained})}}
    For two policies $\pi$ and $\eta~: S\rightarrow A$, the discounted state distributions of the two policies can be bounded like:
    \begin{equation*}
        \|d^\pi - d^\eta\|_1 \leq \dfrac{2\gamma}{1-\gamma}\mathbb{E}_{s\sim d^\eta}\left[D_{TV}(\pi\| \eta)\right].
    \end{equation*}
\end{lemma}

\begin{proof}
    First can transform the discounted state distribution in the vector form to
    \begin{align*}
        d^\pi &= (1-\gamma) \sum_{t=0}^\infty \gamma^t (P^\pi)^t \rho = (1-\gamma)(1-\gamma P^\pi)^{-1}  \rho,
    \end{align*}
    where $\rho$ is the initial state distribution.

    We define the matrices $G := (I - \gamma P^\pi)^{-1}$, $G' := (I - \gamma P^\eta)^{-1}$ and $\Delta := P^\pi - P^\eta$. Then we have:
    \begin{align*}
        G'^{-1} - G^{-1} &= (I - \gamma P^\eta) - (I - \gamma P^\pi) = \gamma(P^\pi - P^\eta) = \gamma \Delta.
    \end{align*}
    left-multiplying by $G$ and right-multiplying by $G'$, we obtain:
    \begin{equation*}
        G - G' = \gamma G' \Delta G.
    \end{equation*}
    Thus
    \begin{align}
        d^\pi - d^\eta &= (1 - \gamma) (G - G') \rho = (1-\gamma) \gamma G' \Delta G \rho = \gamma G' \Delta d^\pi.   \label{state_dist_diff}
    \end{align}

    Then we bound the norm:
    \begin{align*}
        \|d^\pi - d^\eta\|_1 &= \|\gamma G' \Delta d^\pi\|_1 \leq \gamma \|G'\|_1 \|\Delta d^\pi\|_1 \\
    \end{align*}

    $\|G'\|_1$ is bounded by:
    \begin{align}
        \|G'\|_1 = \|(I - \gamma P^\eta)^{-1}\|_1 \leq \sum_{t=0}^\infty \gamma^t \|P^\eta\|_1^t = \sum_{t=0}^\infty \gamma^t \cdot 1 = \dfrac{1}{1-\gamma}.   
    \end{align}

    $\|\Delta d^\pi\|_1$ is bounded by:
    \begin{align*}
        \|\Delta d^\pi\|_1 &= \sum_{s'} | \sum_{s} \Delta(s'|s) d^\pi (s) |\\
        &= \sum_{s'} | \sum_{s} (P^\pi(s'|s) - P^\eta(s'|s)) d^\pi(s) |\\
        &\leq \sum_{s,s'} |P^\pi(s'|s) - P^\eta(s'|s)| d^\pi(s) \\
        &= \sum_{s,s'} \left| \sum_{a} P(s'|s,a) (\pi(a|s) - \eta(a|s)) \right| d^\pi(s)\\
        &\leq \sum_{s,a} \left|\pi(a|s) - \eta(a|s)\right| d^\pi(s) \sum_{s'} P(s'|s,a)\\
        &= \sum_{s} d^\pi(s) \sum_{a} |\pi(a|s) - \eta(a|s)|\\
        &= 2\mathbb{E}_{s\sim d^\pi(\cdot)}\left[D_{TV}\left[\pi(\cdot|s)\|\eta(\cdot|s)\right]\right].
    \end{align*}

    Combining the two terms above we can obtain:
    \begin{align*}
        \|d^\pi - d^\eta\|_1 &\leq \gamma \cdot \dfrac{1}{1-\gamma} \cdot 2\mathbb{E}_{s\sim d^\pi(\cdot)}\left[D_{TV}\left[\pi(\cdot|s)\|\eta(\cdot|s)\right]\right]\\
        &= \dfrac{2\gamma}{1-\gamma} \mathbb{E}_{d^\pi}\left[D_{TV}\left[\pi(\cdot|s)\|\eta(\cdot|s)\right]\right].
    \end{align*}
\end{proof}

\begin{lemma}
    \label{lemma:value_diff}
    The value difference of two policies $\pi, \eta$ staring from the same state $\tilde{s}$ is bounded as follows:
    \begin{equation*}
        \left|V^\pi(\tilde{s}) - V^\eta(\tilde{s})\right| \leq \dfrac{2R_{\text{max}}}{(1-\gamma)^2}~\mathbb{E}_{d^\pi(\cdot|s_0=\tilde{s})}\left[D_{TV}\left[\pi(\cdot|s)\|\eta(\cdot|s)\right]\right].
    \end{equation*}
\end{lemma}

\begin{proof}
    We can obtain:
    \begin{align*}
        &\left|V^\pi(\tilde{s}) - V^\eta(\tilde{s})\right| \\
        &= \left|\dfrac{1}{1-\gamma}\left(\mathbb{E}_{\rho^\pi(\cdot|s_0=\tilde{s})}\left[r(s,a)\right] - \mathbb{E}_{\rho^\eta(\cdot|s_0=\tilde{s})}\left[r(s,a)\right]\right)\right|\\
        &= \dfrac{1}{1-\gamma}\left|\sum_{s,a} \left(\rho^\pi(s,a|s_0=\tilde{s}) - \rho^\eta(s,a|s_0=\tilde{s})\right) r(s,a)\right|\\
        &\leq \dfrac{R_{\text{max}}}{1-\gamma} \sum_{s,a} \left|\rho^\pi(s,a|s_0=\tilde{s}) - \rho^\eta(s,a|s_0=\tilde{s})\right| \\
        &= \dfrac{R_{\text{max}}}{1-\gamma} \sum_{s,a} \Bigl|d^\pi(s|s_0=\tilde{s})\pi(a|s) - d^\eta(s|s_0=\tilde{s})\eta(a|s)\Bigr|\\
        &= \dfrac{R_{\text{max}}}{1-\gamma} \sum_{s,a} \Bigl|d^\pi(s|s_0=\tilde{s})\pi(a|s) - d^\pi(s|s_0=\tilde{s})\eta(a|s) + d^\pi(s|s_0=\tilde{s})\eta(a|s) - d^\eta(s|s_0=\tilde{s})\eta(a|s)\Bigr|\\ 
        &\leq \dfrac{R_{\text{max}}}{1-\gamma} \Bigl[ \sum_{s,a} \left| \pi(a|s) -\eta(a|s) \right| d^\pi(s|s_0=\tilde{s}) + \sum_{s,a} \left| d^\pi(s|s_0=\tilde{s}) -d^\eta(s|s_0=\tilde{s}) \eta(a|s)\right| \Bigr]\\
        &= \dfrac{R_{\text{max}}}{1-\gamma} \Bigl[ 2\mathbb{E}_{d^\pi(\cdot|s_0=\tilde{s})}\left[ D_{TV}\left[\pi(\cdot|s)\|\eta(\cdot|s)\right]\right] + \sum_{s,a} \left|d^\pi(s|s_0=\tilde{s}) - d^\eta(s|s_0=\tilde{s})) \eta(a|s)\right| \Bigr]\\
        &\leq \dfrac{R_{\text{max}}}{1-\gamma} \left[ 2\mathbb{E}_{d^\pi(\cdot|s_0=\tilde{s})}\left[ D_{TV}\left[\pi(\cdot|s)\|\eta(\cdot|s)\right]\right] + \|d^\pi(\cdot|s_0=\tilde{s}) - d^\eta(\cdot|s_0=\tilde{s})\|_1 \|\eta\|_\infty \right] \qquad~~ \text{(Holder's)}\\
        &\leq \dfrac{R_{\text{max}}}{1-\gamma} \left[ 2\mathbb{E}_{d^\pi(\cdot|s_0=\tilde{s})}\left[ D_{TV}\left[\pi(\cdot|s)\|\eta(\cdot|s)\right]\right] + \|d^\pi(\cdot|s_0=\tilde{s}) - d^\eta(\cdot|s_0=\tilde{s})\|_1 \right] \qquad\qquad (\|\eta\|_\infty\leq 1)
    \end{align*}

    By setting the initial state distribution $\rho$ in Lemma~\ref{lemma:state_dis_diff} to an one-hot vector with only one at the position $\tilde{s}$, we can apply Lemma~\ref{lemma:state_dis_diff} and derive:
    \begin{align*}
        \left|V^\pi(\tilde{s}) - V^\eta(\tilde{s})\right| &\leq \dfrac{R_{\text{max}}}{1-\gamma} \left[ 2\mathbb{E}_{d^\pi(\cdot|s_0=\tilde{s})}\left[ D_{TV}\left[\pi(\cdot|s)\|\eta(\cdot|s)\right]\right] + \dfrac{2\gamma}{1-\gamma} \mathbb{E}_{d^\pi(\cdot|s_0=\tilde{s})}\left[D_{TV}\left[\pi(\cdot|s)\|\eta(\cdot|s)\right]\right] \right]\\
        &= \dfrac{2R_{\text{max}}}{(1-\gamma)^2}~\mathbb{E}_{d^\pi(\cdot|s_0=\tilde{s})}\left[D_{TV}\left[\pi(\cdot|s)\|\eta(\cdot|s)\right]\right].
    \end{align*}
\end{proof}

\section{Experimental Settings}
\label{app:setting}

\begin{figure}[t]
    \centering
    \includegraphics[width=\textwidth]{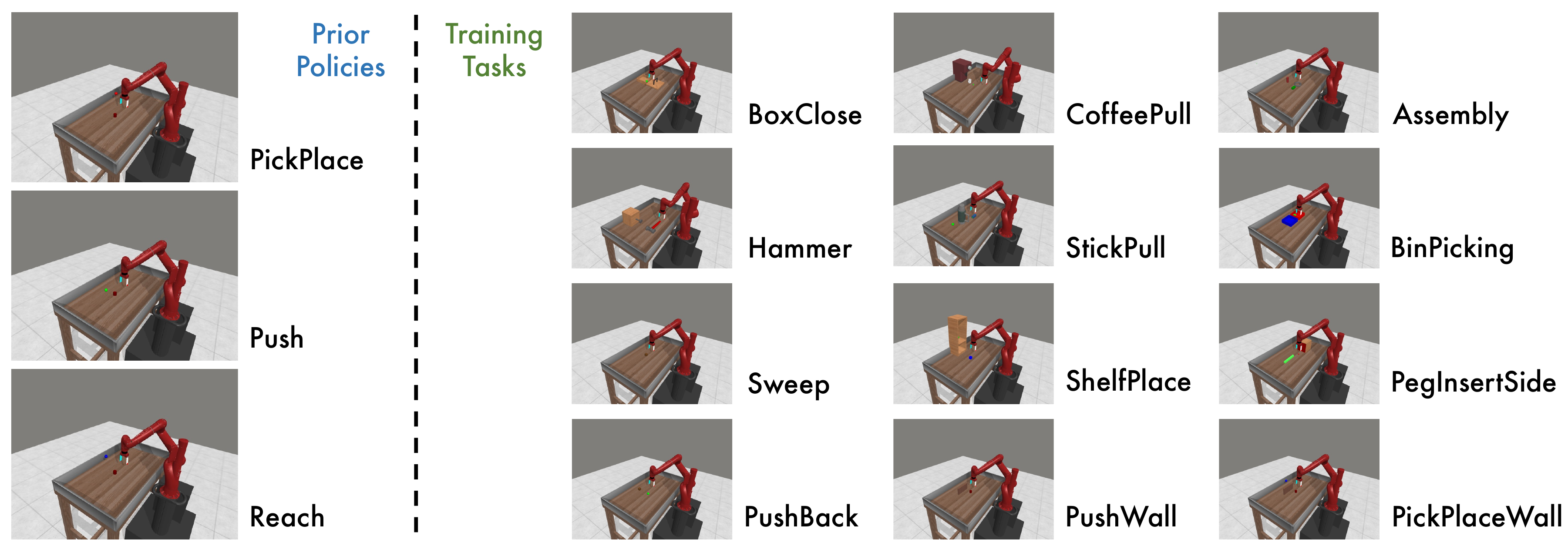}
    \caption{The prior policies and training tasks of MetaWorld experiments in Section~\ref{sec:exp:basic_results}.}
    \label{fig:metaworld_env_illustration}
\end{figure}

\begin{figure}[t]
    \centering
    \includegraphics[width=\textwidth]{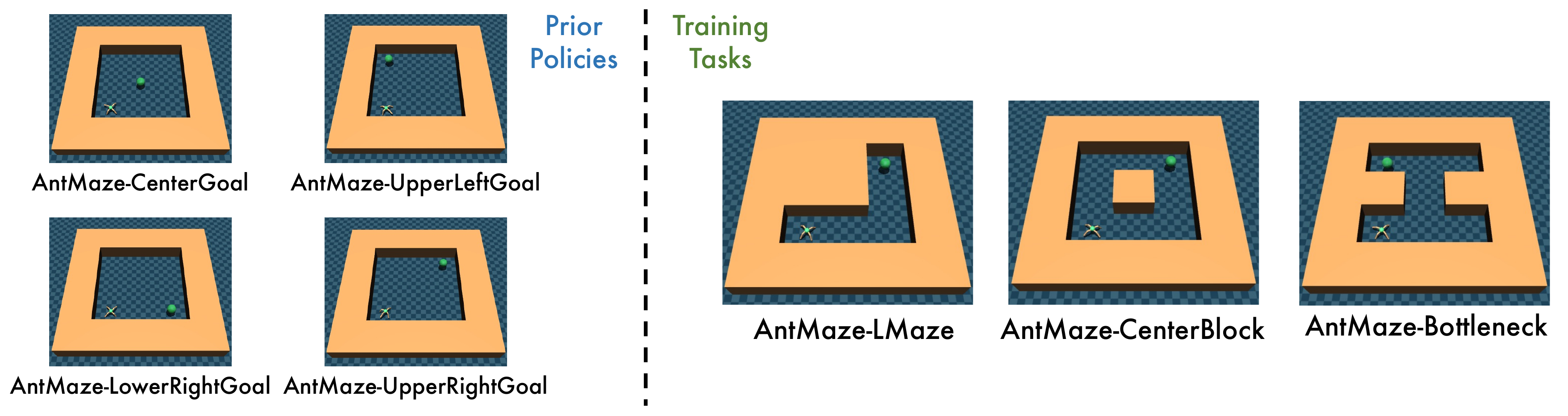}
    \caption{The prior policies and training tasks of AntMaze experiments in Section~\ref{sec:exp:basic_results}. The green ball represents the goal position across all environments.}
    \label{fig:antmaze_env_illustration}
\end{figure}

\subsection{Environment Setting Details}
\label{app:setting:env}

In this section, we provide details on the environments. 

\textbf{MetaWorld}~
For the experiments in Section~\ref{sec:exp:basic_results}, we use three prior policies~(\textit{Push}, \textit{Reach}, \textit{PickPlace}) that are trained with SAC~\cite{haarnoja2018soft} and can perform $100\%$ success rate in the corresponding tasks. We use 12 tasks different from those of the prior policies for the downstream tasks. Following settings in prior works~\cite{zhangcup,yang2020multi}, we randomly reset the goal positions at the start of each episode. The illustration of the environments is shown in Figure~\ref{fig:metaworld_env_illustration}. To provide rigorous analysis, we examine the zero-shot performances of all prior policies in the downstream tasks. As shown in Figure~\ref{fig:zero-shot_performance_mix}~(Left), the prior policies hardly perform high-return behaviors in most cases. For experiments in Section~\ref{sec:exp:identify}, we use four sets of prior policies:~Set 1~(\textit{Push}, \textit{Reach}, \textit{PickPlace}),~Set 2~(\textit{Push}, \textit{Reach}, \textit{PickPlace}, \textit{Hammer}, \textit{PushBack}, \textit{ShelfPlace}),~Set 3~(\textit{Push}, \textit{Reach}, \textit{PickPlace}, \textit{Hammer}, \textit{PushBack}, \textit{ShelfPlace}, \textit{a sub-optimal task policy}),~Set 4~(\textit{Push}, \textit{Reach}, \textit{PickPlace}, \textit{Hammer}, \textit{PushBack}, \textit{ShelfPlace}, \textit{a sub-optimal task policy}, \textit{an optimal task policy}). 

\textbf{AntMaze}~
For the experiments in Section~\ref{sec:exp:basic_results}, we use 4 prior policies trained for approaching different goals in an empty grid and 3 downstream tasks with diverse maze layouts and goals, as shown in Figure~\ref{fig:antmaze_env_illustration}. The agent obtains a reward of $100$ only if the agent reaches the goal. The zero-shot performances of the prior policies in each downstream task are shown in Figure~\ref{fig:zero-shot_performance_mix}~(Right). The results validate that the prior policies can not achieve the desired goals in all tasks due to the mismatched maze layouts. 

\begin{figure}[h]
    \centering
    \includegraphics[width=\textwidth]{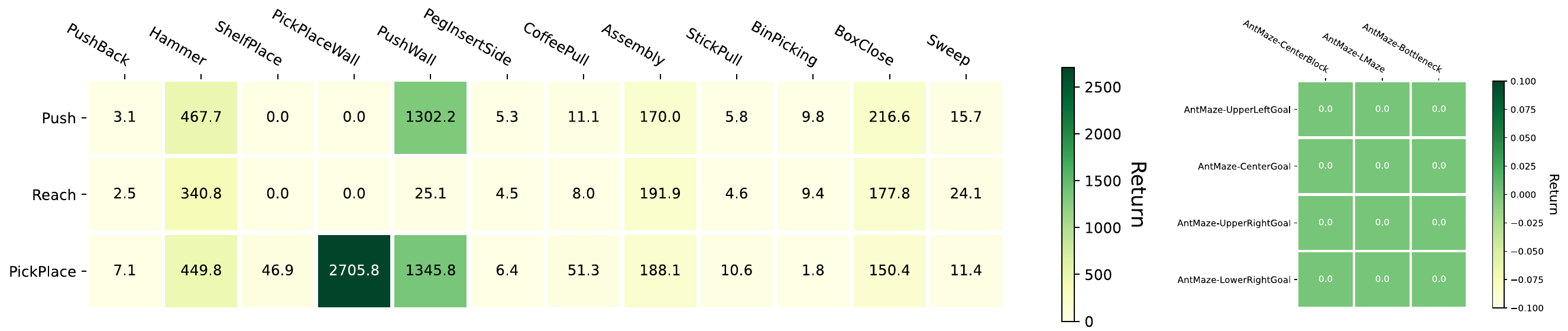}
    \caption{(\textit{Left})~The zero-shot performance of the prior policies in 12 downstream tasks of MetaWorld. The results demonstrate that the prior policies rarely perform high-return behaviors in the downstream tasks. (\textit{Right})~The zero-shot performance of the prior policies in 3 downstream tasks of AntMaze. The results demonstrate that no prior policies can approach the desired goals due to the mismatched maze layouts.}
    \label{fig:zero-shot_performance_mix}
\end{figure}

\begin{table}[t]
    \centering
    \caption{Hyperparameter configuration, which is shared across all runs of SMEC.}
    \vspace{0.6em}
    \begin{tabular}{cc}
        \toprule
        \textbf{Hyperparameter} & \textbf{Value} \\
        \midrule
        Number of hidden layers (Policy) & 3  \\
        Number of hidden units per layer (Policy) & 400  \\
        Number of hidden layers (Value) & 3  \\
        Number of hidden units per layer (Value) & 400 \\
        Learning rate & $3e^{-4}$ \\
        Batch size & 128 \\
        Temperature coefficient & Auto \\
        Target smoothing coefficient & $5e^{-3}$ \\
        Policy training delay & 2 \\
        Warm-start steps & $5e^4$ \\
        Exploration length $h$ & $50$ (MetaWorld) / $70$ (AntMaze)\\
        Truncation constant $\epsilon$ & $1e^{-4}$ \\
        UCB coefficient $c$ & $10$ (MetaWorld) / $1$ (AntMaze) \\
        \bottomrule
    \end{tabular}
    \label{tab:hyperparameter}
\end{table}

\subsection{Algorithm Implementation Details}
\label{app:setting:alg}
\textbf{SMEC}:~We use soft-actor critic~(SAC)~\cite{haarnoja2018soft} as our backbone algorithm. To train the value function proposed in Section~\ref{sec:method:evaluation}, we sample a batch of transitions from the replay buffer and compute the bellman targets for each policy following~(\ref{eq:value_loss_hybrid}). As for the interactions, we estimate the value of each policy via the target Q functions every $h$ step. We use the max value of the two target Q functions for better exploration. The policy with the maximum value estimation is adopted for the interactions in the subsequent $h$ steps. Before adding the value-guided behavior planning, we use a random policy before $5e^4$ steps for warmup exploration.

\textbf{AC-Teach}~\cite{kurenkov2020ac}:~
We adapt the backbone algorithm of the original implementation to soft-actor critic for fair comparisons. We use the same hyper-parameters as the original paper, using commitment decay $\phi=0.99$ and commitment threshold $\beta=0.6$.

\textbf{QMP}~\cite{zhang2023efficient}:~
QMP performs behavior sharing by considering all policies' actions and selecting the best one evaluated via the Q function. The behavior policy can be formulated as follows:
$$
\eta(a|s) := \delta_{a = {\arg\max}_{a \sim \{\pi(\cdot|s)\}\cup\{\mu_i(\cdot|s)\} } Q_\theta(s,a)},
$$
where $\delta$ denotes the Dirac delta distribution, and $Q_\theta$ denotes the behavior value function of the current task. As for our implementation, we query the actions of all policies and perform the value-guided action selection at each step, the same as the original implementation. 

\textbf{SkillS}~\cite{vezzani2022skills}:~
SkillS proposes to sequence the policies via a hierarchical architecture. We adapt the original implementation to the soft-actor critic version and train the meta controller with the policy gradient with soft Q estimations. For sample-efficient training, we perform data augmentation for the meta controller same as the original paper. 

\textbf{CUP}~\cite{zhangcup}:~
CUP trains the task policy by performing regularization to better actions proposed by prior policies. We use the same hyper-parameters by setting $\beta_1$ as $30$ and $\beta_2$ as $0.003$. Unlike the original implementation using vector environments for data collection, we use a single environment like the other algorithms and introduce the policy regularization after $100k$ steps. 

\textbf{MAMBA}~\cite{cheng2020policy}:~
MAMBA performs policy gradient updates by introducing the baseline value functions. We use the same hyperparameter configurations as the original implementation.

\textbf{MultiPolar}~\cite{barekatain2021multipolar}:~
MultiPolar aggregates the actions of all prior policies through an aggregation function and an auxiliary function. We use the open-sourced policy implementation and train the functions mentioned above with soft-actor critic algorithm.

\subsection{Hyper-parameter Details}
\label{app:setting:hyper}

Table~\ref{tab:hyperparameter} lists the hyperparameters used in our experiments. The hyperparameters related to SAC are shared across all baseline methods. 

\section{Addtional Experimental Results}
\label{app:addition}

\subsection{Utilization of Prior Policies}
\label{app:addition:full_utilization}

\begin{figure}[H]
    \centering
    \includegraphics[width=\textwidth]{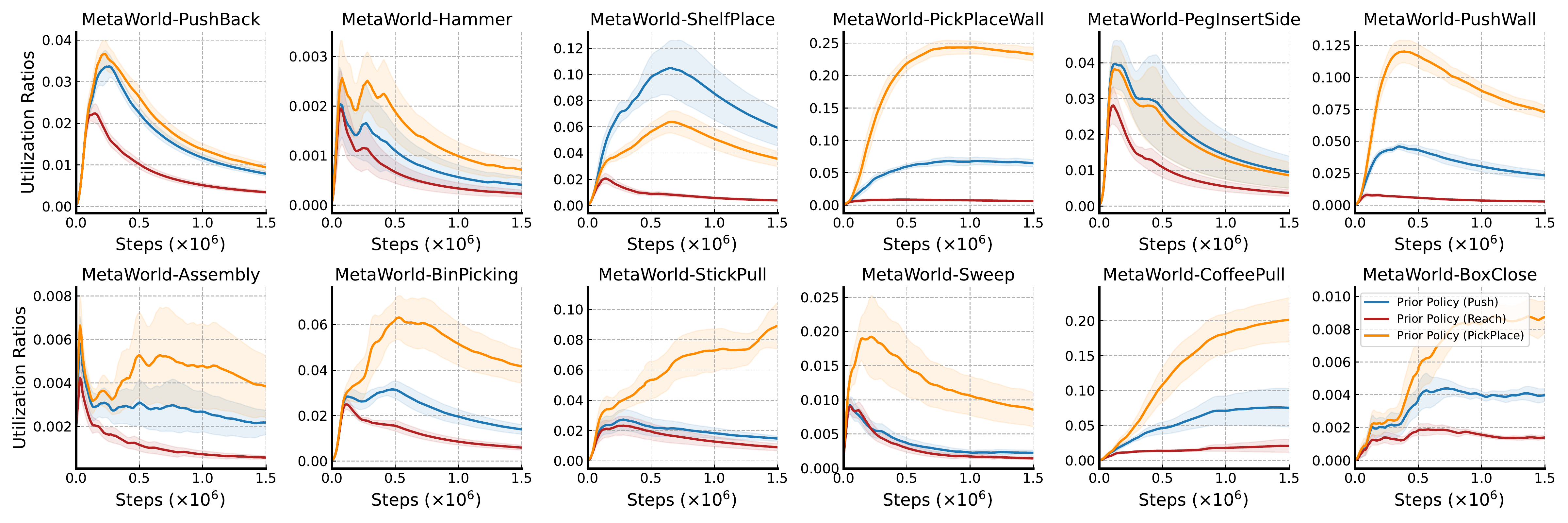}
    \caption{Utilization ratios of prior policies induced by SMEC in all MetaWorld tasks.}
    \label{fig:policy_utilization_metaworld}
\end{figure}
\begin{figure}[H]
    \centering
    \includegraphics[width=0.6\textwidth]{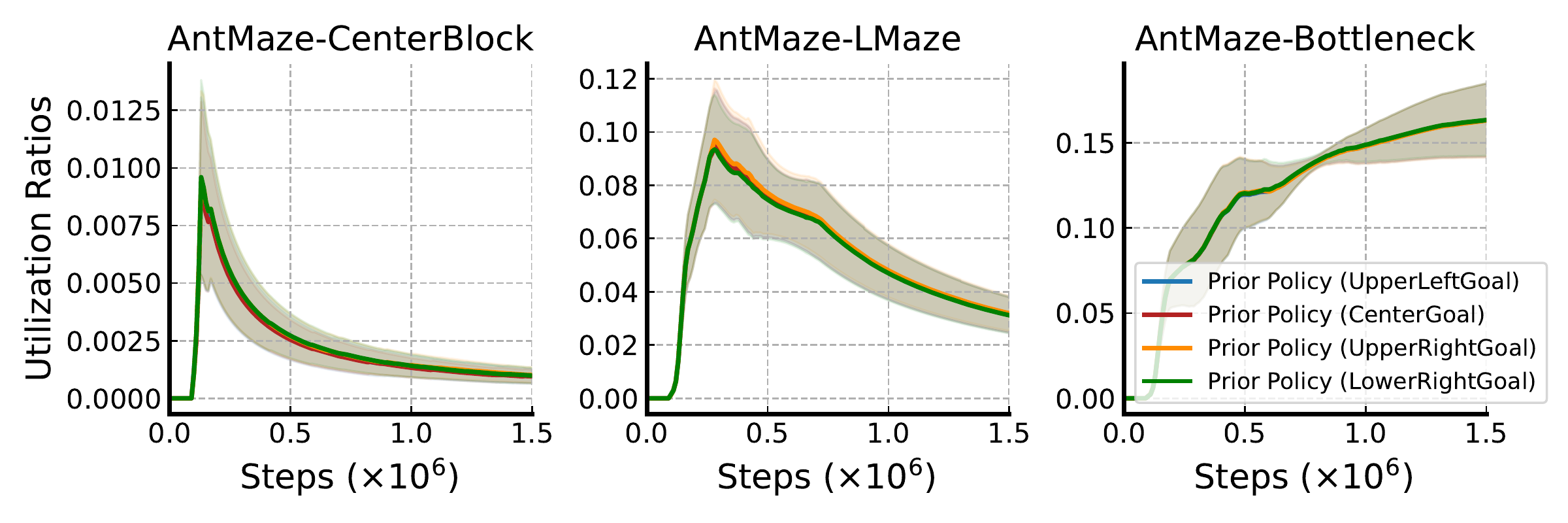}
    \caption{Utilization ratios of prior policies induced by SMEC in all AntMaze tasks.}
    \label{fig:policy_utilization_antmaze}
\end{figure}

We show the utilization ratios of our method across all tasks of MetaWorld and AntMaze in Figure~\ref{fig:policy_utilization_metaworld} and Figure~\ref{fig:policy_utilization_antmaze}, respectively. The utilization of the prior policies gradually increases at the early training stages, which guides the agent to form temporally-extended behaviors. As the training proceeds, the utilization of prior policies decreases thanks to the performance improvement of the task policy. 

\subsection{Investigation on Different Prior Policies}
\label{app:addition:more_priors}

\begin{figure}[t]
    \centering
    \includegraphics[width=0.95\textwidth]{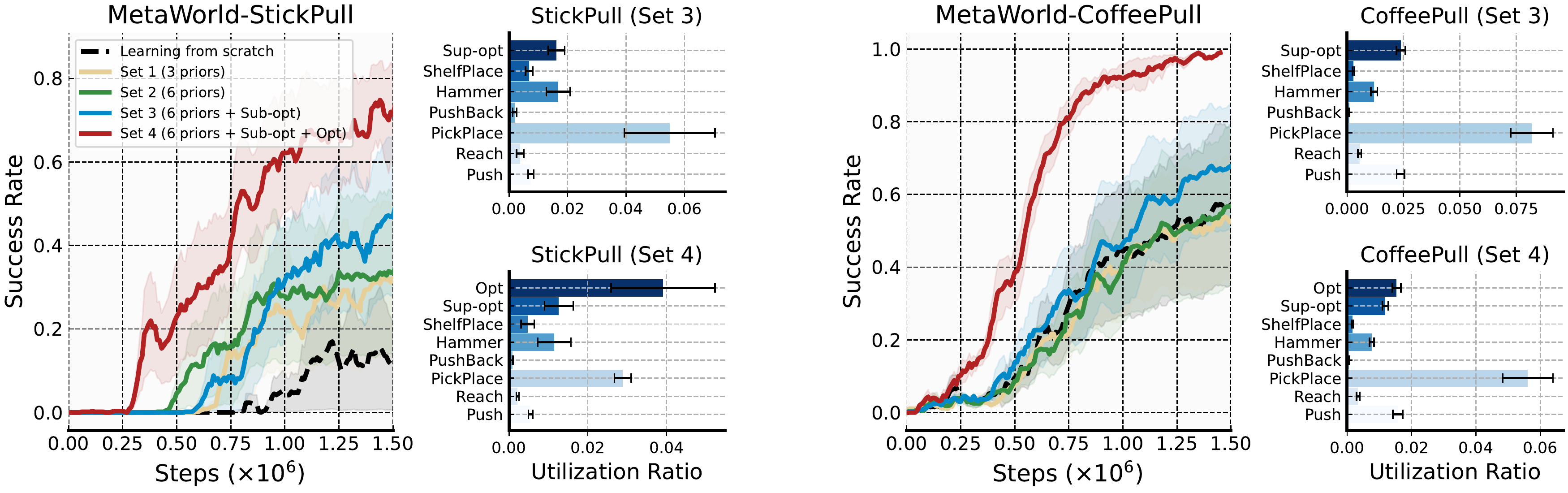}
    \caption{The learning curves and utilization ratios of prior policies in StickPull~(\textit{Left}) and CoffeePull~(\textit{Right}). "Sub-opt" denotes the sup-optimal task policy, and "Opt" denotes the optimal task policy.}
    \label{fig:more_priors_two_env}
\end{figure}
Intuitively, the training efficiency differs with different sets of prior policies. When there are related prior policies concerning the current task, the training efficiency would be enhanced with the assistance of the prior policies. To validate whether SMEC can identify and fully exploit the related prior policies, we perform experiments with different prior policy sets on two challenging MetaWorld tasks. Specifically, we set four different prior policy sets: Set 1~(Push, Reach, PickPlace); Set 2~(Push, Reach, PickPlace, PushBack, Hammer, ShelfPlace); Set 3~(Push, Reach, PickPlace, PushBack, Hammer, ShelfPlace, a sub-optimal task policy); Set 4~(Push, Reach, PickPlace, PushBack, Hammer, ShelfPlace, a sub-optimal task policy, an optimal task policy). The learning curves and prior utilization ratios are shown in Figure~\ref{fig:more_priors_two_env}, which show that the performance is significantly boosted by introducing the optimal task policy as one of the prior policies. Furthermore, our method identifies and maximally exploits the optimal policy in StickPull. While the optimal policy in CoffeePull is not the most selected, the resulting performance demonstrates that the optimal policy is rationally exploited to guide learning.

\subsection{Detailed Results of Ablation Studies}
\label{app:addition:ablation}

\begin{figure}[H]
    \centering
    \includegraphics[width=0.98\textwidth]{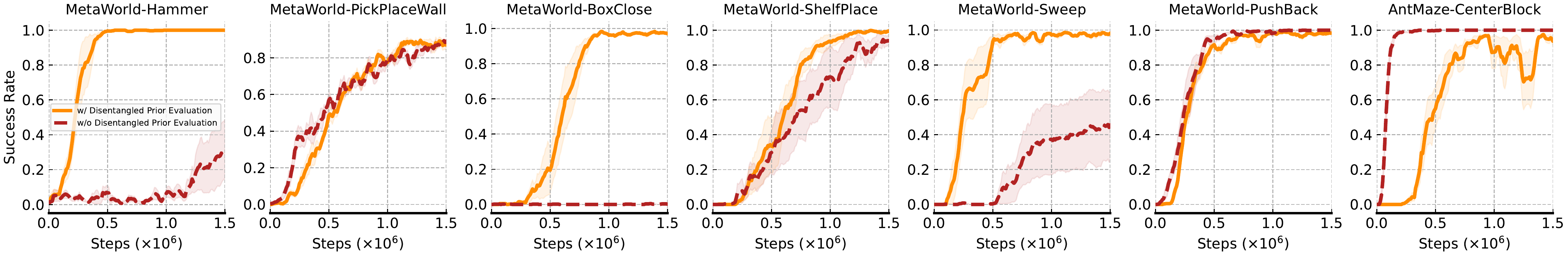}
    \caption{Full ablation results on the disentangled prior policy evaluation.}
    \label{fig:ablation_prior_evaluation_full}
\end{figure}
\textbf{Evaluation of prior policies.}~
We perform experiments on multiple tasks to validate the effect of the disentangled evaluation of prior policies. The detailed results on several tasks and the aggregated results are shown in Figure~\ref{fig:ablation_prior_evaluation_full}. The results demonstrate that the disentangled prior evaluation is crucial in 3 out of 7 tasks. The value estimation of the single behavior value function is inconsistent with the allocated policy, which would result in overly using the prior policies and thus hinder learning.

\begin{figure}[H]
    \centering
    \includegraphics[width=0.98\textwidth]{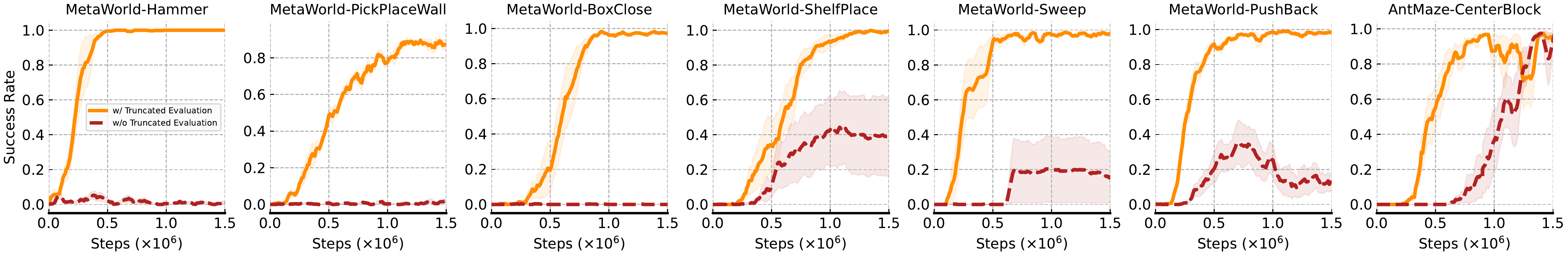}
    \caption{Full ablation results on the truncated behavior evaluation.}
    \label{fig:ablation_truncated_evaluation_full}
\end{figure}

\textbf{Short horizon evaluation.}~
We perform experiments on multiple tasks to validate the effect of truncated behavior evaluation. The detailed results on several tasks and the aggregated results are shown in Figure~\ref{fig:ablation_truncated_evaluation_full}. The results validate that the evaluation of the truncated behaviors is essential for efficient training across all seven tasks.

\begin{figure}[H]
    \centering
    \includegraphics[width=0.98\textwidth]{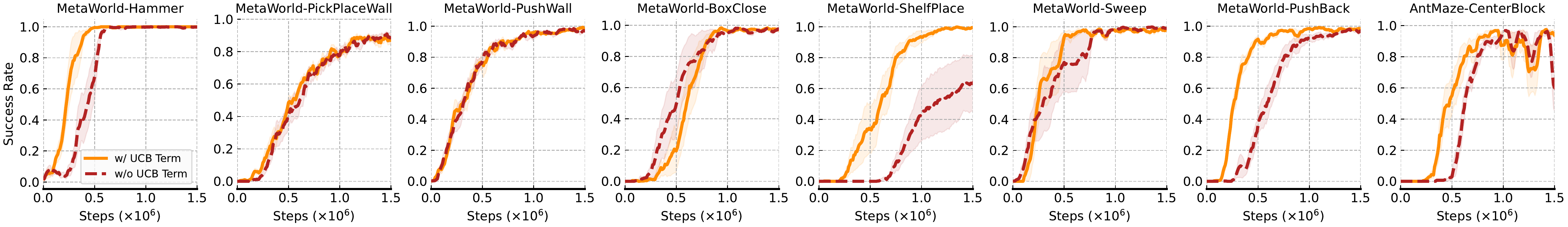}
    \caption{Full ablation results on the UCB term.}
    \label{fig:ablation_ucb_term_full}
\end{figure}
\textbf{UCB-term}
We perform experiments on multiple tasks to validate the effect of UCB term for policy selection. The detailed results on several tasks and the aggregated results are shown in Figure~\ref{fig:ablation_ucb_term_full}. The results demonstrate that the UCB-term can improve the sample efficiency in 3 out of 8 tasks.

\begin{figure}[H]
    \centering
    \includegraphics[width=0.85\textwidth]{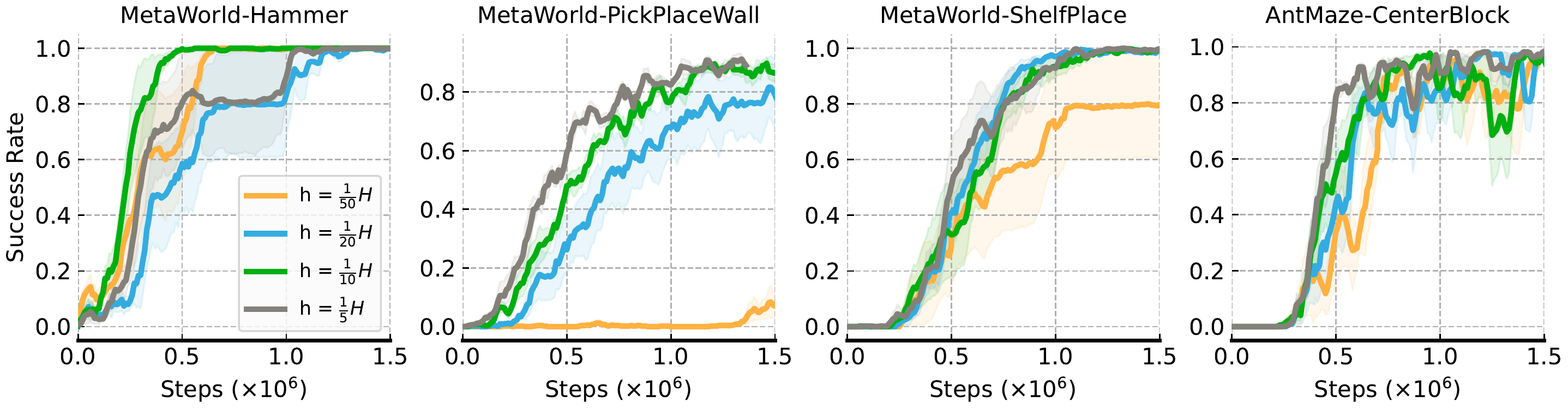}
    \caption{Full ablation results on the behavior length $h$.}
    \label{fig:ablation_exploration_length_full}
\end{figure}
\begin{figure}[H]
    \centering
    \includegraphics[width=0.98\textwidth]{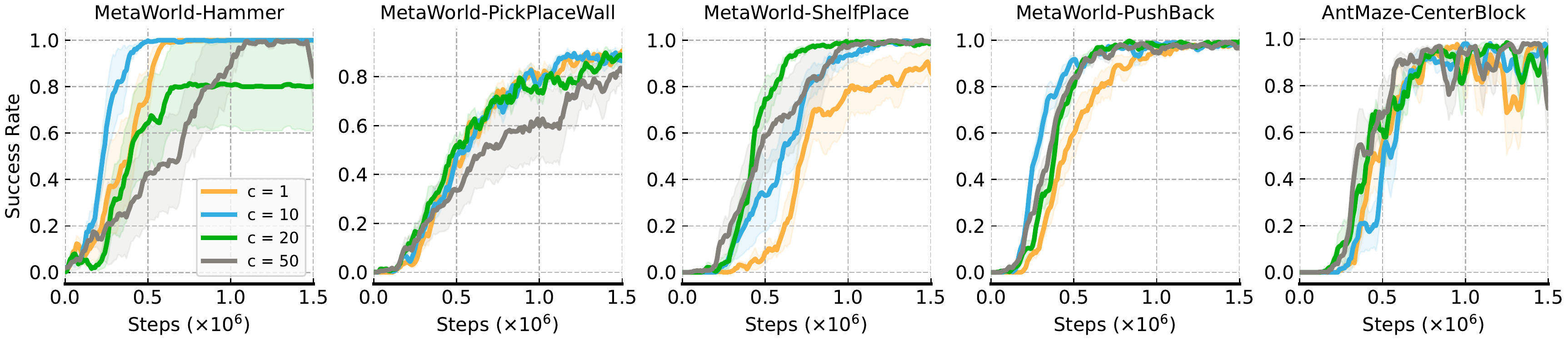}
    \caption{Full ablation results on the UCB coefficient $c$.}
    \label{fig:ablation_ucb_coeff_full}
\end{figure}

\textbf{Hyperparameter sensitivity}~
We perform ablation experiments on two hyper-parameters: behavior length $h$ and UCB coefficient $c$. As for the behavior length $h$, we compare the variants with four different values~(\textit{i.e.},~$H/50,~H/20,~H/10,~H/5$, and $H$ denotes the episode length of the environments), and the detailed results are shown in Figure~\ref{fig:ablation_exploration_length_full}, which demonstrate that the performance difference is not significant when the behavior length is sufficiently long. As for the UCB coefficient $c$, we use four different values (\textit{i.e.},~$1,~5,~20,~50$). The results are shown in Figure~\ref{fig:ablation_ucb_coeff_full}, which show that the UCB coefficient shows a minor impact on the performance. 

\subsection{Accuracy of Value Estimation on Prior Policies}
\label{app:addition:accuracy}

\begin{figure}[H]
    \centering
    \includegraphics[width=\textwidth]{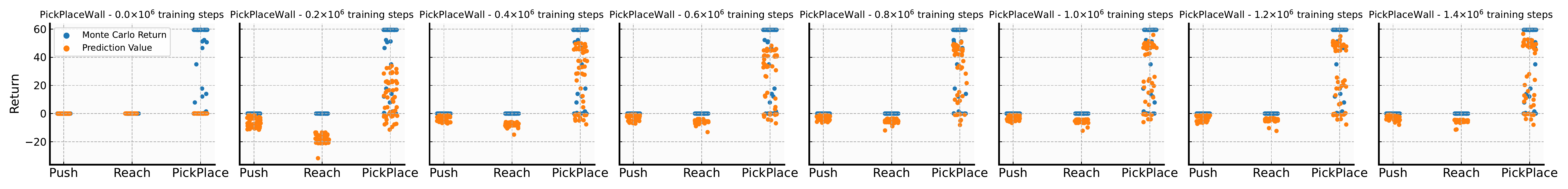}
    \caption{Comparisons between the value estimations of the prior policies and ground truth Monte Carlo returns along the training steps in PickPlaceWall.}
    \label{fig:value_acc_pickplacewall}
\end{figure}

\begin{figure}[H]
    \centering
    \includegraphics[width=\textwidth]{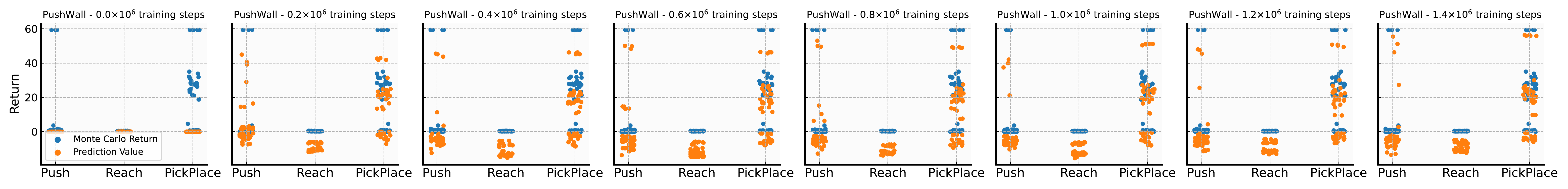}
    \caption{Comparisons between the value estimations of the prior policies and ground truth Monte Carlo returns along the training steps in PushWall.}
    \label{fig:value_acc_pushwall}
\end{figure}

In this subsection, we aim to investigate the accuracy of the learned value estimations of the prior policies, which plays a crucial role in validating the effect of our method. Specifically, we reload the checkpoints of the value function and compare the prediction values with the Monte Carlo returns of the prior policies. We show the results in PickPlaceWall and PushWall in Figure~\ref{fig:value_acc_pickplacewall} and Figure~\ref{fig:value_acc_pushwall}, respectively. The results indicate that the learned value functions can accurately estimate the return induced by the corresponding prior policies.

\subsection{Comparison with Random Policy Section}
\label{app:addition:random}

\begin{figure}[H]
    \centering
    \includegraphics[width=0.98\textwidth]{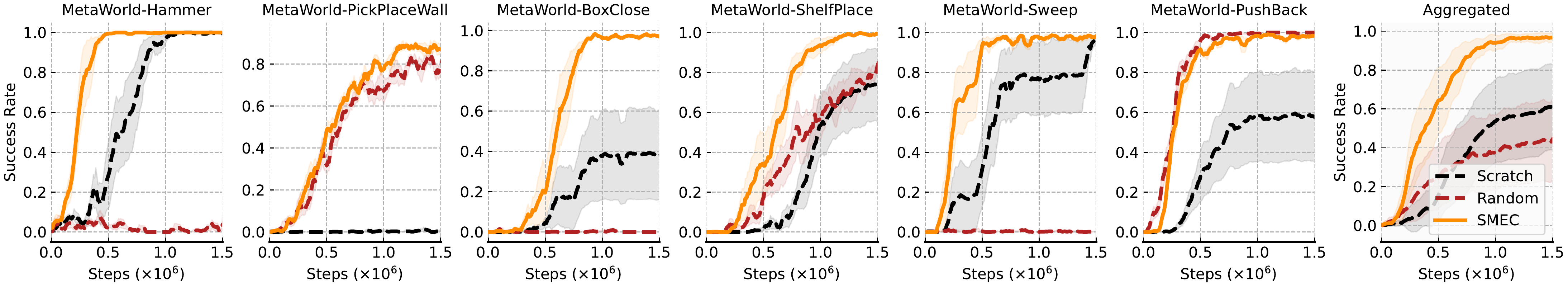}
    \caption{Comparison with the algorithm with random policy section. The rightmost plot denotes the aggregated results. \textit{Random} denotes the algorithm performing random policy selection per $h$ steps, and \textit{Scratch} denotes learning from scratch without prior policies. The results demonstrate that SMEC outperforms \textit{Random} and \textit{Random} even underperforms \textit{Scratch}, indicating value-guided behavior planning plays an essential role in the superior performance of SMEC.}
    \label{fig:ablation_random_full}
\end{figure}

This subsection aims to investigate the impact of value-guided behavior planning in reinforcement learning. To do so, we compare the performance of SMEC to a variant that randomly selects the behavior policy from the set of policies~(\textit{i.e.},~$\{\pi\}\cup\{\mu_i\}_{i=1}^K$) every $h$ steps. We present the results of this comparison in Figure~\ref{fig:ablation_random_full}, which shows that the variant barely works in three out of six tasks while nearly matching SMEC in the remaining tasks. We hypothesize that the variant's performance is closely related to the utilities of the prior policies with respect to the current task. Specifically, when the prior policies are highly relevant to the current task, the most relevant prior policy can provide near-optimal behaviors. Thus, randomly selecting the policy can yield sufficient data for efficient learning. However, if all the prior policies are irrelevant to the current task, randomly selecting the policy can impede learning. In contrast, as demonstrated by our experiments, SMEC can adaptively select the policy that leads to the most promising behaviors, resulting in consistently superior performance in all six tasks.

\subsection{Discussion on the Effect of $\bar{\gamma}$}
\label{app:addition:discussion_normal_gamma}

\begin{figure}[H]
    \centering
    \includegraphics[width=\textwidth]{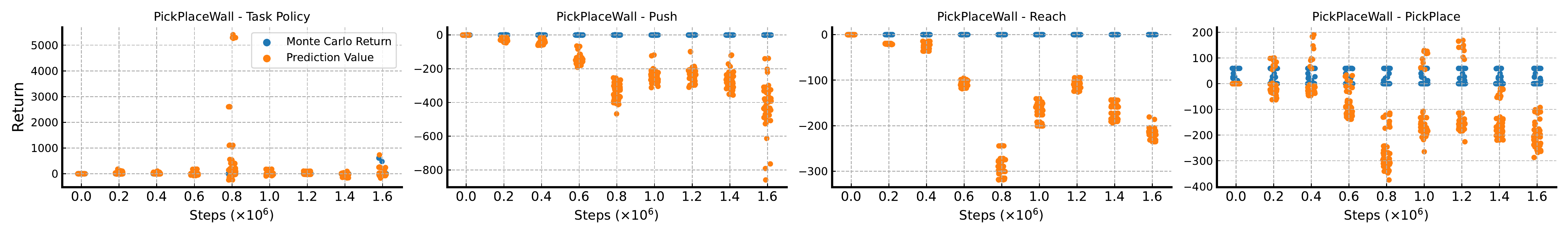}
    \caption{Under the algorithm variant using non-truncated horizon for evaluating prior policies~(\textit{i.e.},~$\bar{\gamma}=\gamma$), we compare the value estimations of all the policies and the ground truth Monte Carlo returns along the training steps in PickPlaceWall.}
    \label{fig:value_acc_pickplacewall_normal_gamma}
\end{figure}

The results presented in Figure~\ref{fig:effect_factor_multiple}~(Middle) and Figure~\ref{fig:ablation_truncated_evaluation_full} demonstrate a significant performance degradation when the horizon used for evaluating prior policies is not truncated~(\textit{i.e.},$\bar{\gamma}=\gamma$). We plot the value estimations throughout training to investigate the reasons behind these failures, as shown in Figure~\ref{fig:value_acc_pickplacewall_normal_gamma}. Our observations reveal that the value functions of the prior policies provide inaccurate predictions when using the non-truncated horizon. Moreover, the value estimations are unstable throughout training, leading to an underestimation of the value. We believe the phenomenon results from the severe off-policyness problem, \textit{i.e.},~the training data for the value function of the prior policies mostly comes from the task policy. However, the problem is circumvented by using the truncated horizon as shown in Figure~\ref{fig:value_acc_pickplacewall}, which can benefit from the regularization effect induced by the lower discount factor~\cite{jiang2015dependence,amit2020discount}. The unstable value estimations, in turn, can further influence policy selection and lead to performance degradation.

\subsection{Analysis of Computation Cost}
\label{app:addition:computation}

\begin{table}[t]
    \small
    \centering
    \caption{The training wall-clock time (hours) for 1 million training steps of the algorithms using SAC as the backbone algorithm. We report the mean and std of the wall-clock time across five runs with different random seeds.}
    \vspace{0.6em}
    \begin{tabular}{c|c|cccccc}
        \toprule
         Task & Scratch & MultiPolar & AC-Teach & CUP & SkillS & QMP & SMEC \\
         \midrule
         BoxClose & $6.43_{\pm0.14}$ & $8.09_{\pm0.27}$ & $8.26_{\pm0.21}$ & $8.38_{\pm0.18}$ & $14.73_{\pm0.05}$ & $7.74_{\pm0.32}$ & $8.01_{\pm0.08}$ \\
        \bottomrule
    \end{tabular}
    \label{tab:computation}
\end{table}

This subsection provides an analysis of the computation cost to investigate the computation efficiency of the algorithms. For the experiment on MetaWorld-BoxClose in Section~\ref{sec:exp:basic_results}, we compute the training wall clock times for 1 million training steps of several algorithms. The results shown in Table~\ref{tab:computation} demonstrate that SMEC outperforms four out of five baselines concerning the wall-clock efficiency. Compared with the \textit{Scratch} that learns without the prior policies, SMEC only takes about $24.6\%$ more wall-clock time to run the same number of environment steps.

\begin{table}[t]
    \small
    \centering
    \caption{The training wall-clock time (hours) for 1 million training steps of SMEC with different numbers of prior policies. We report the mean and std of the wall-clock time across five runs with different random seeds.}
    \vspace{0.6em}
    \begin{tabular}{c|c|cccc}
        \toprule
         Task & Scratch & 3 Prior policies & 6 Prior policies & 7 Prior policies & 8 Prior policies \\
         \midrule
         StickPull & $5.92_{\pm0.11}$ & $7.27_{\pm0.13}$ & $7.29_{\pm0.02}$ & $7.32_{\pm0.09}$ & $7.57_{\pm0.11}$\\
        \bottomrule
    \end{tabular}
    \label{tab:computation_different_priors}
\end{table}

Furthermore, we analyze the required computation cost of SMEC given different numbers of the prior policies. We compare the computation cost of the experiments in Section~\ref{sec:exp:identify}. The number of prior policies varies across the experiments (3/6/7/8). The computation cost of training SMEC under different prior policy settings is demonstrated in Table~\ref{tab:computation_different_priors}, which shows that the required additional computation cost of SMEC is acceptable~(only $27.8\%$ more wall-clock time than \textit{Scratch} under the 8 prior policies case).

\section{Extended Experiments on Continual Learning}
\label{app:continual}

\begin{figure}[H]
    \centering
    \includegraphics[width=\textwidth]{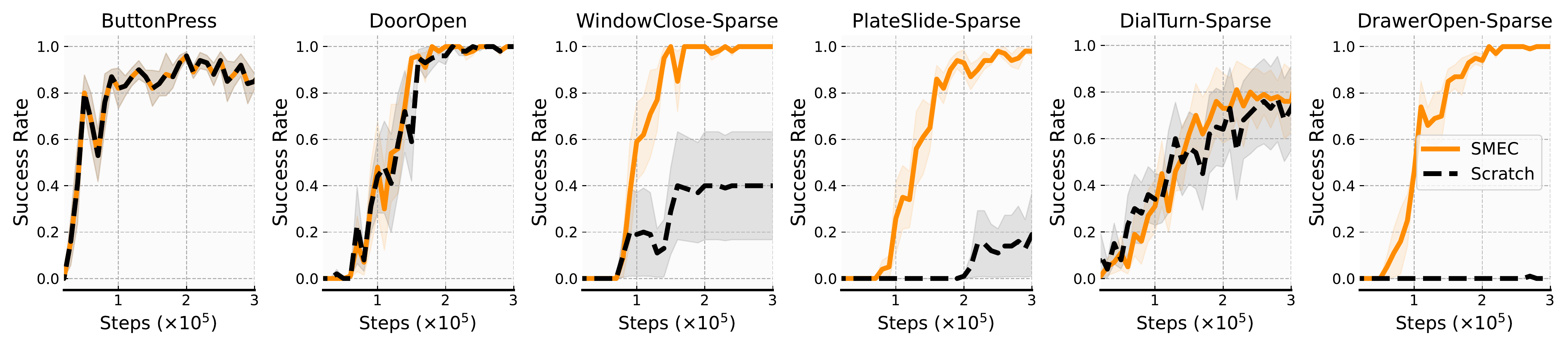}
    \caption{The results of continual learning experiments. We compare SMEC with \textit{Scratch} that learns without using the previous policies. The solid line and shaded regions represent the mean and standard deviation across five runs with different random seeds.}
    \label{fig:continal_with_scratch}
\end{figure}

Efficiently reusing previously learned policies is an appealing ability in various settings, especially in cases where abundant prior policies are available. In this section, we examine the effectiveness of SMEC in the Continual Reinforcement Learning setting~\cite{wolczyk2022disentangling}, where the cumulated policies are reused for efficient learning in the current task. 

We propose a sequence of 6 tasks from MetaWorld as a continual learning case (ButtonPress~$\rightarrow$~DoorOpen~$\rightarrow$~WindowClose~$\rightarrow$~PlaceSlide~$\rightarrow$~DialTurn~$\rightarrow$~DrawerOpen). Furthermore, we convert the dense reward functions of the last 4 tasks to the sparse variants that only provide non-zero rewards $1$ if the agent succeeds in the task. The two fundamental problems in continual learning are \textit{preventing catastrophic forgetting} (\textit{i.e.},~preventing the performance degradation of the policy concerning the previous tasks) and \textit{increasing forward transfer} (\textit{i.e.}, speeding up the learning by reusing knowledge from previous tasks). Since we only investigate whether our method can be helpful to speed up learning by reusing previously learned policies, we exclude the catastrophic forgetting problem by utilizing individual policy modules for each task. We launch an individual SAC algorithm and reuse all the previously learned policies for each task.

The results shown in Figure~\ref{fig:continal_with_scratch} demonstrate that SMEC is effective in the Continual Learning case by reusing all previous policies. Especially in the sparse reward tasks where learning from scratch hardly makes any progress in 3 out of 4 tasks, SMEC learns efficiently by exploiting the previous policies, which indicates SMEC is applicable and effective in the setting. 

\end{document}